\documentclass[11pt]{article}
\usepackage[preprint]{tmlr}
\setcitestyle{numbers,square,citesep={,}}
\usepackage{amsfonts,amsthm,mathtools,bm}
\usepackage{graphicx,xcolor,booktabs,array,tabularx,multirow}
\usepackage{microtype,placeins,stfloats}
\usepackage[breaklinks,hidelinks]{hyperref}
\usepackage[capitalize]{cleveref}

\newcommand{\R}{\mathbb R}
\newcommand{\dd}{\,\mathrm d}

\newtheorem{theorem}{Theorem}
\newtheorem{proposition}[theorem]{Proposition}

\theoremstyle{definition}

\title{Volterra Generative Models via Markovian Lifting}
\author{
  \name Bingyan Han \email bingyanhan@hkust-gz.edu.cn \\
  \addr The Hong Kong University of Science and Technology (Guangzhou)}
\hypersetup{pdftitle={Volterra Generative Models via Markovian Lifting},
  pdfauthor={Bingyan Han}}

\begin{document}
\maketitle
\begin{abstract}
	Diffusion-based generative models usually perturb data through Markov processes driven by Brownian motion. Volterra processes allow more flexible temporal dependence, but their non-Markovian structure complicates score learning and reverse-time simulation. We develop Volterra generative models and provide Markovian approximations. Gaussian conditioning then reduces the denoising problem to the original data dimension. We present the connections with variance-exploding and variance-preserving representations. Our framework also permits reuse of pretrained denoisers such as EDM2 networks. With one auxiliary factor, our models trained from scratch achieve FIDs of $0.547$ on MNIST and $1.860$ on conditional CIFAR-10. With autoguidance, ImageNet64 FID reaches $1.912$ after fine-tuning, while ImageNet512 FID reaches $1.903$ using frozen pretrained networks.
\end{abstract}

\section{Introduction}
\label{sec:introduction}

Diffusion generative models construct a forward noising process that gradually transforms data into a tractable reference distribution, and then to learn the reverse transformation from noise back to data. Early diffusion probabilistic models implemented this idea through discrete noising chains \cite{sohl2015deep,ho2020denoising}, while score-based models learned gradients of perturbed data densities and used them for sampling \citep{song2019generative}. The continuous-time formulation of \citet{song2020score} unifies these views through stochastic differential equations (SDEs). Subsequent developments have improved conditional generation through classifier and classifier-free guidance \cite{dhariwal2021diffusion,ho2022classifierfree}, reduced the computational cost of high-resolution synthesis through latent diffusion \cite{rombach2022high}, and accelerated sampling through specialized solvers for diffusion ODEs \cite{lu2022dpm}. EDM further clarified the roles of denoiser preconditioning, noise schedules, and numerical integration \cite{karras2022elucidating}, while EDM2 improved training dynamics through a network design that controls activation and weight magnitudes \cite{karras2024analyzing}. Other developments include diffusion transformers for scalable image generation \cite{peebles2023scalable} and consistency models for generation in one or a few steps \cite{song2023consistency}.

Most standard diffusion models use Brownian-driven Markov processes for forward corruption. This choice provides explicit conditional distributions and tractable reverse-time dynamics, but fixes a particular form of temporal dependence. A broader motivation comes from Mandelbrot and Van Ness \cite{mandelbrot1968fractional}, who studied fractional Brownian motion to describe fluctuations with scaling and dependence beyond those of ordinary Brownian motion. In that family, the Hurst parameter $H\in(0,1)$ controls temporal regularity and increment dependence: $H=1/2$ gives Brownian motion, smaller values produce rougher paths, and larger values produce smoother paths with persistent increments. Inspired by this flexibility, we study fractional kernels within a general Volterra construction for the forward perturbation.

A Volterra process weights past Brownian increments through a convolution kernel. The kernel therefore determines how earlier perturbations contribute to the current state. Such constructions have been studied extensively in stochastic analysis and quantitative finance \cite{carmona1998fractional,gatheral2018volatility,harms2019affine}. Their general non-Markovian structure, however, prevents a direct application of standard time-reversal formulas to the image state alone. Finite exponential kernel approximations provide a practical solution by introducing auxiliary variables whose joint evolution is Markovian \cite{bayer2023markovian,alfonsi2024approximation}. Working directly with the convolution kernel also makes approximation errors transparent: the It\^o isometry relates the stochastic error to a deterministic $L^2$ kernel error.

Previous generative models have explored alternative perturbations in several ways. Yoon et al.~\cite{yoon2023score} use stable L\'evy processes to introduce heavy-tailed noise. The generative fractional diffusion model (GFDM) of Nobis et al.~\cite{nobis2024generative} is more closely related to our work: it uses a Markovian approximation of fractional Brownian noise and reduces the learned score to the image dimension. We build on this reduction and study how kernel approximation, conditional denoising, and numerical implementation fit together. In particular, the conditional residual provides both a manageable learning objective and a precise connection to existing Gaussian diffusion models.

Our main contributions are as follows.
\begin{enumerate}
	\item We formulate forward perturbation through Volterra processes and bound the approximation error directly by the $L^2$ kernel error. This gives a simple criterion for constructing finite exponential approximations. A signed projection accommodates fractional kernels in both the rough ($H<1/2$) and smooth ($H>1/2$) regimes. We also tackle covariance singularity with a positive lower time cutoff for numerical stability.
	
	\item We give several reasons for using the conditional residual. First, it reduces the learned part of the augmented score to the original data dimension, while the auxiliary Gaussian contribution is available analytically. Second, in the drift-free isotropic setting, the normalized residual has the full path law of a time-changed variance-exploding diffusion, and a further rescaling gives a variance-preserving representation. Third, we also establish a correspondence with EDM preconditioning \cite{karras2022elucidating}, which allows pretrained EDM2 networks \cite{karras2024analyzing} to be reused without adding auxiliary input channels. 
	
	\item In the numerical study, we consider only one auxiliary process, which has small approximation errors already. Models trained from scratch achieve FIDs of $0.547$ on MNIST and $1.860$ on conditional CIFAR-10. The CIFAR-10 configuration uses one tenth as many optimizer updates and $60\%$ fewer image presentations than the GFDM configuration \cite{nobis2024generative}, with a comparable model size. Using pretrained EDM2 networks and autoguidance \cite{karras2024autoguidance}, ImageNet64 reaches $1.912$ after $1{,}000$ fine-tuning updates, and ImageNet512 reaches $1.903$ without additional training. Our experiments use at most four GPUs, making these experiments feasible with limited computational resources.
\end{enumerate}

The rest of the paper is organized as follows. Section~\ref{sec:approximation} introduces the Volterra formulation, kernel approximation, and Markovian representation. Section~\ref{sec:denoising} develops residual denoising and its connections with VE and VP diffusions. Section~4 examines covariance degeneracy and the choice of a positive lower time cutoff. Section~\ref{sec:experiments} presents the numerical results, and Section~\ref{sec:conclusion} concludes the paper. The appendix contains proofs, further analysis of reverse-time sampling, literature comparisons, and experimental settings.

\section{Volterra processes}
\label{sec:approximation}

Suppose the real image $\mathbf X_0 \in \R^p$ follows the distribution $P_0$ with a finite second moment. Denote $\mathbf B$ as a standard $p$-dimensional Brownian motion, independent of $\mathbf X_0$. Volterra processes are given by
\begin{equation}
 \begin{aligned}
 \mathbf X_t={}&\alpha(t)\mathbf X_0
 +\int_0^t G(t-s)\mathbf b(s)\dd s + \int_0^t G(t-s)\boldsymbol\Sigma(s)\dd\mathbf B_s,
 \qquad t\in[0,T],
 \end{aligned}
 \label{eq:volterra_time}
\end{equation}
where $G\in L^2(0,T)$ is a scalar kernel, $\mathbf b:[0,T]\to\R^p$ and $\boldsymbol\Sigma:[0,T]\to\R^{p\times p}$ are bounded measurable deterministic functions, and $\alpha$ is continuous with $\alpha(0)=1$. For simplicity, we set $\alpha(t)=e^{-\lambda t}$ with $\lambda > 0$. It controls the decay of the initial data signal. The kernel $G$ specifies how past drift and Brownian increments contribute to the current state. To capture the roughness, a typical choice of $G$ is the fractional kernel
\begin{equation}
	G_H(u)=\frac{u^{H-1/2}}{\Gamma(H+1/2)},\qquad H\in(0,1).
	\label{eq:fractional}
\end{equation} 
The constant $H$ is called the Hurst parameter, which controls the regularity of the perturbation. The case $H=1/2$ corresponds to the classical Brownian case. When $H< 1/2$, the kernel is singular near the origin and produces rougher perturbations. When $H>1/2$, the kernel is smoother near the origin and produces more persistent dependence. Naturally, we wonder which choice of $H$ will lead to better generative models.

In general, $\mathbf X_t$ is not Markovian since it relies on the path. Direct simulation is difficult, and applying a diffusion time-reversal formula to $\mathbf X$ alone is generally unjustified. We therefore consider Markovian approximations of Volterra processes.

\subsection{Markovian approximation}

For an approximation kernel $\widehat G\in L^2(0,T)$, let $\widehat{\mathbf X}_t$ be defined by \eqref{eq:volterra_time} with $G$ replaced by $\widehat G$. Comparing with the fractional BM framework in \cite{nobis2024generative}, it is direct to quantify the $L^2$ approximation error of processes under the Volterra formulation, as shown in Proposition \ref{prop:kernel-process}. The proof is given in Appendix~\ref{app:proj:densityerror}.
\begin{proposition}
\label{prop:kernel-process}
The approximation error satisfies
\begin{equation}
 \begin{aligned}
 & \sup_{0\leq t\leq T}\mathbb E\|\mathbf X_t-\widehat{\mathbf X}_t\|^2 \leq
 \left(T\|\mathbf b\|_\infty^2+
 \|\boldsymbol\Sigma\|_{\infty,\mathrm{HS}}^2\right) \|G-\widehat G\|_{L^2(0,T)}^2.
 \end{aligned}
 \label{eq:forward-error}
\end{equation}
The matrix norm is the Hilbert--Schmidt norm, and the essential supremum is over time.
\end{proposition}

To obtain Markovian approximations for $\mathbf X$, we can choose the following weighted sum of exponential kernels:
\begin{equation}\label{eq:kernel}
 \widehat G(u)=\chi e^{-\lambda u} +\sum_{i=1}^K\psi_i e^{-\kappa_i u}.
\end{equation}
Suppose the rates $\lambda, \kappa_i$ are non-negative and distinct. The weight parameters $\chi$ and $\psi_i$ can be negative. Under mild technical conditions, we can show that $\widehat G$ approximates the fractional kernel $G_H$ arbitrarily well in $L^2(0, T)$.

Proposition \ref{prop:kernel-process} suggests that we can find a good approximation kernel by minimizing the error in \eqref{eq:forward-error}. A standard approach is to fix $\lambda, \chi, \kappa_i$ in advance, such as the geometric rates in \cite{nobis2024generative}. Then other coefficients $\psi := (\psi_1, \ldots, \psi_K)$ have an explicit least-squares solution. 
\begin{proposition}\label{prop:projection}
Let $A_{ij}=\int_0^T e^{-(\kappa_i+\kappa_j)u}\dd u$ and $r_i=\int_0^T[G(u)-\chi e^{-\lambda u}]e^{-\kappa_i u}\dd u$. For every $\delta\geq0$, the unique minimizer of
\begin{equation}
 \|G-\widehat G\|_{L^2(0,T)}^2 + \delta\| \psi \|^2
 \label{eq:projection-objective}
\end{equation}
is $\psi=(A + \delta I)^{-1}r$. Its squared kernel error is
\begin{equation}
 \|G-\chi e^{-\lambda\cdot}\|_{L^2}^2 - 2\psi^\top r+\psi^\top A \psi.
 \label{eq:projection}
\end{equation}
\end{proposition}
Appendix~\ref{app:proj:projection} gives the detailed proof. Since the rates $\lambda, \kappa_i$ are distinct, the matrix $A$ is positive definite. Moreover, we can impose the ridge penalty with coefficient $\delta > 0$ for numerical stability. In Appendix~\ref{app:exp-implementation}, we present the selection of parameters $\lambda, \chi, \kappa_i$ in the approximation kernel before training.

\subsection{Projected dynamics}
An advantage of the exponential approximation is that the process $\widehat{\mathbf X}$ admits a finite-dimensional Markovian representation through the augmented state $\mathbf Z=(\widehat{\mathbf X},\mathbf Y^1,\ldots,\mathbf Y^K)$. Indeed, $\widehat{\mathbf X}$ satisfies
\begin{align}
 \dd\widehat{\mathbf X}_t
 &=\Bigl[-\lambda\widehat{\mathbf X}_t
       +\sum_i(\lambda-\kappa_i)\psi_i\mathbf Y_t^i
       +\varepsilon\mathbf b(t)\Bigr]\dd t + \varepsilon\boldsymbol\Sigma(t)\dd\mathbf B_t,
 \qquad\widehat{\mathbf X}_0=\mathbf X_0.
 \label{eq:general_primary_realization}
\end{align}
where $\varepsilon := \chi + \sum_i\psi_i$ and the auxiliary processes are
\begin{align}
	 \dd\mathbf Y_t^i
	&=(-\kappa_i\mathbf Y_t^i+\mathbf b(t))\dd t
	+\boldsymbol\Sigma(t)\dd\mathbf B_t,
	\quad\mathbf Y_0^i=0.
	\label{eq:general_aux_realization}
\end{align}
A proof is given in Appendix~\ref{app:proj:realization}. We note that every process here is driven by the same Brownian motion.  From now on, we choose $$\mathbf b=0, \quad \boldsymbol\Sigma(t)=g(t)I_p,$$
where $g(t)$ is a scalar function to be specified later.

Stack the factors as a column vector
\[
\mathbf Y_t=
\begin{pmatrix}
	\mathbf Y_t^1\\ \vdots\\ \mathbf Y_t^K
\end{pmatrix}\in\mathbb R^{Kp},
\qquad
\mathbf Z_t=
\begin{pmatrix}
	\widehat{\mathbf X}_t\\
	\mathbf Y_t^1\\ \vdots\\ \mathbf Y_t^K
\end{pmatrix}\in\mathbb R^{(K+1)p}.
\]
For each data coordinate $j=1,\ldots,p$, define the
column vector $Z_t^{(j)}=( \widehat X_t^{(j)}, Y_t^{1,(j)}, \ldots, Y_t^{K,(j)})^\top \in\mathbb R^{K+1}$. It satisfies
\begin{equation}
	\dd Z_t^{(j)}
	=FZ_t^{(j)}\dd t+g(t)q\dd B_t^{(j)}.
	\label{eq:lift_matrix_F_q}
\end{equation}
Here, $q := (\varepsilon,1,\ldots,1)^\top$, $B_t^{(j)}$ is a scalar coordinate of $\mathbf B_t$, and the matrix $F \in \R^{(K+1) \times (K+1)}$ is defined as
\begin{align*}
	F_{11} &= -\lambda, \quad F_{1, i+1} = (\lambda-\kappa_i)\psi_i, \quad	F_{i+1, i + 1}  =-\kappa_i, \qquad i = 1, \ldots, K,
\end{align*}
with all remaining entries of $F$ zero. For any $j\in\{1,\ldots,p\}$, define the innovation covariance $C(t):=\operatorname{Cov}\bigl(Z_t^{(j)} \mid\mathbf X_0\bigr)$. It is independent of $j$ and satisfies
\begin{equation}
	C'(t)=FC(t)+C(t)F^\top+g(t)^2qq^\top,
	\quad C(0)=0.
	\label{eq:covariance_ode}
\end{equation}
Its blocks are
\begin{align}
	[C_{yy}(t)]_{ij}
	&=\int_0^t e^{-(\kappa_i+\kappa_j)(t-s)}g(s)^2\dd s,
	\label{eq:Cyy_unanchored_main}\\
	[c_{yx}(t)]_i
	&=\int_0^t e^{-\kappa_i(t-s)} \widehat G(t - s) g(s)^2\dd s,
	\label{eq:cyx_unanchored_main}\\
	c_{xx}(t)&=\int_0^t \widehat G(t - s)^2g(s)^2\dd s.
	\label{eq:cxx_unanchored_main}
\end{align}

\section{Residual denoising}
\label{sec:denoising}

Learning the joint score for $\mathbf Z_t$ directly would enlarge the network input from $p$ to $(K+1)p$ coordinates. To avoid this cost, we borrow the Gaussian conditioning idea from \cite{nobis2024generative}. The score function for auxiliary processes $\mathbf Y_t$ can be evaluated analytically. Then the denoising problem is still in the original data dimension $p$.

\subsection{Gaussian conditioning}
To present the Gaussian conditioning formula, we introduce
\begin{align*}
	\eta(t) & := C_{yy}(t)^{-1}c_{yx}(t), \quad	c_{x\mid y}(t)  := c_{xx}(t)-c_{yx}(t)^\top\eta(t), \quad \sigma(t)  := \sqrt{c_{x\mid y}(t)}, \quad \sigma_{\rm eff}(t) := \frac{\sigma(t)}{\alpha(t)}.
\end{align*}

The following Proposition \ref{prop:covariance} provides the well-posedness on conditional variances. The proof is deferred to Appendix~\ref{app:proj:covariance}.
\begin{proposition}\label{prop:covariance}
	Suppose that $g(t) > 0$ is continuous, $\chi \ne 0$, and $\lambda,\kappa_1,\ldots,\kappa_K$ are distinct. Then $C(t)\succ0$ and $c_{x\mid y}(t)>0$ for every $t>0$. Moreover, $\sigma_{\rm eff}(t)$ is continuous and strictly increasing from zero, with
	\begin{equation*} 
		\sigma_{\rm eff}(t)^2
		=\chi^2\min_{a\in\R^K}\int_0^t g(s)^2
		\left(e^{\lambda s}-\sum_i a_i e^{\kappa_i s}\right)^2\dd s.
	\end{equation*}
\end{proposition}

Next, define the residual image by
\begin{equation}
 \boldsymbol\xi_t := \widehat{\mathbf X}_t - \sum_i\eta_i(t)\mathbf Y_t^i.
 \label{eq:residual}
\end{equation}
We can show that the residual has the following representation in distributions:
\begin{equation}
 \boldsymbol\xi_t \overset{\rm law}= \alpha(t)\mathbf X_0 + \sigma(t)\boldsymbol\zeta_t,
 \label{eq:residual-law}
\end{equation}
where the random variable $\boldsymbol\zeta_t \sim \mathcal N(0, I_p)$ and is independent of $\mathbf X_0$ and $\mathbf Y_t$. 

We train a denoiser $D_\theta$ by
\begin{equation}
 \mathcal L(\theta) := \mathbb{E}_{t}
 \mathbb{E}_{\mathbf{X}_0}
 \mathbb{E}_{\boldsymbol\xi_t \mid\mathbf{X}_0}\left[
       w(t) \|D_\theta(\boldsymbol\xi_t,t)-\mathbf X_0\|^2\right],
 \label{eq:loss}
\end{equation}
where $t$ is sampled independently of the data and Gaussian noise, and the weight $w(t)>0$. Conditional on $t$, the innovation is sampled with its prescribed covariance. Suppose the loss in \eqref{eq:loss} is finite for at
least one measurable denoiser. Among denoisers with finite
loss, its minimizer is $$D_t^*(\boldsymbol\xi)=\mathbb E[\mathbf X_0\mid\boldsymbol\xi_t=\boldsymbol\xi].$$ 
If we include class labels in denoiser inputs, all identities also hold. 

To approximate this population minimizer $D_t^*(\boldsymbol\xi)$, we use denoisers with EDM preconditioning~\cite{karras2022elucidating}. Define $$m(t)^2=\sigma(t)^2+\alpha(t)^2\sigma_{\rm data}^2.$$ 
With a backbone network $f_\theta$, the denoiser is given as follows:
\begin{equation}
 \begin{aligned}
 D_\theta(\boldsymbol\xi,t)
 &=c_{\rm skip}(t)\boldsymbol\xi+c_{\rm out}(t)\,
 f_\theta(c_{\rm in}(t)\boldsymbol\xi,c_{\rm noise}(t)),\\
 c_{\rm skip}(t) &=\frac{\alpha(t)\sigma_{\rm data}^2}{m(t)^2},
 \quad c_{\rm out}(t) =\frac{\sigma(t) \sigma_{\rm data}}{m(t)}, \quad c_{\rm in}(t) =\frac{1}{m(t)}, \quad c_{\rm noise}(t) =\tfrac14\log (\sigma_{\rm eff}(t) ).
 \end{aligned}
 \label{eq:precondition}
\end{equation}
The data-scale parameter is $\sigma_{\rm data}=0.5$.

The reduction to a data-dimensional residual $\boldsymbol\xi$ also permits the reuse of a pretrained denoiser with ordinary Gaussian perturbations. Indeed, Proposition \ref{prop:transfer} shows that the denoiser is asked to solve the same Gaussian estimation problem after division by $\alpha(t)$.  Appendix~\ref{app:gen-transfer-section} provides the proof. It permits us to reuse EDM2 pretrained weights. Our ImageNet experiments use native EDM2 networks~\cite{karras2024analyzing}. ImageNet512 retains their latent representation and fixed decoder.
\begin{proposition}
\label{prop:transfer}
Let $$D^{\rm G,*}(\mathbf u,s)=\mathbb E[\mathbf X_0\mid\mathbf X_0+s\boldsymbol\zeta=\mathbf u]$$ for an independent standard Gaussian $\boldsymbol\zeta$. Then
\begin{equation}
 D_t^*(\boldsymbol\xi)
 =D^{\rm G,*}(\boldsymbol\xi/\alpha(t),\sigma_{\rm eff}(t)).
 \label{eq:population-transfer}
\end{equation}
For every backbone $f_\theta$, \eqref{eq:precondition} is related to the EDM network $D_\theta^{\rm EDM}$ as follows:
\begin{equation}
 D_\theta(\boldsymbol\xi,t)
 =D_\theta^{\rm EDM}(\boldsymbol\xi/\alpha(t),\sigma_{\rm eff}(t)).
 \label{eq:network-transfer}
\end{equation}
\end{proposition}

\subsection{Augmented score matching}
The reverse dynamic requires the score function of $\mathbf{Z}_t$. Let $p_t^\xi$ be the density of $\boldsymbol\xi_t$, and let $q_t^Y$ denote the Gaussian density of the stacked factors $\mathbf Y_t$. The joint density satisfies
\begin{equation}
 p_t(\mathbf x,\mathbf y)=
 p_t^\xi\left(\mathbf x-\sum_i\eta_i(t)\mathbf y^i\right)q_t^Y(\mathbf y).
 \label{eq:factorization}
\end{equation}

Differentiating the logarithm of \eqref{eq:factorization}, we can prove the following augmented score decomposition result.
\begin{proposition}
	\label{prop:score_decomposition_unanchored_main}
	Under the preceding assumptions, the augmented score $\nabla_z\log p_t(\mathbf x,\mathbf y)$ is give by
	\begin{equation}
		\begin{pmatrix}
			\mathbf s^*(\boldsymbol\xi,t)\\
			-\eta_1(t)\mathbf s^*(\boldsymbol\xi,t)\\
			\vdots\\
			-\eta_K(t)\mathbf s^*(\boldsymbol\xi,t)
		\end{pmatrix}
		+\begin{pmatrix}
			\mathbf 0\\
			-(C_{yy}(t)^{-1}\otimes  I_p)\mathbf y
		\end{pmatrix},
		\label{eq:score_decomp_unanchored_main}
	\end{equation}
	where $\mathbf s^*(\boldsymbol\xi, t) = \nabla_{\boldsymbol\xi}\log p_t^\xi(\boldsymbol\xi)$ and $\boldsymbol\xi=\mathbf x-\sum_i\eta_i(t)\mathbf y_i$.
\end{proposition}
Since we only need to estimate the residual score function $\mathbf s^*(\boldsymbol{\xi}, t)$, the score estimation is performed in $p$ dimensions only. In particular, we can show that
\begin{equation}
\nabla_{\boldsymbol\xi}\log p_t^\xi(\boldsymbol\xi)
 =\frac{\alpha(t)D_t^*(\boldsymbol\xi)-\boldsymbol\xi}{c_{x\mid y}(t)}.
 \label{eq:residual-score}
\end{equation}
The conditional expectation $D_t^*(\boldsymbol\xi)$ is estimated by $D_\theta(\boldsymbol\xi, t)$.

\subsection{Connections with VE and VP diffusions}
Previously, we introduce the residual $\boldsymbol\xi$ mainly for reducing dimensions of denoiser inputs and reusing the pretrained weights. Interestingly, it is also connected with the standard variance-exploding (VE) and variance-preserving (VP) constructions~\cite{song2020score}. Define
\begin{equation}
 \mathbf U_t=\boldsymbol\xi_t/\alpha(t),
 \qquad v(t)=\sigma_{\rm eff}(t)^2.
 \label{eq:normalized-residual}
\end{equation}
Proposition \ref{prop:ve} gives the VE representation of the normalized residual, with a time-changed Brownian motion.
\begin{proposition}\label{prop:ve}
The normalized residual $\mathbf U_t$ satisfies
\begin{equation}
 (\mathbf U_t)_{0\leq t\leq T}
 \overset{\rm law}= (\mathbf X_0+\mathbf W_{v(t)})_{0\leq t\leq T},
 \quad \mathbf W\perp\mathbf X_0,
 \label{eq:ve-path}
\end{equation}
where $\mathbf W$ is a standard $p$-dimensional Brownian motion. Equivalently, in law,
\begin{equation}
 \dd\mathbf U_t=\sqrt{v'(t)}\dd\widetilde{\mathbf W}_t,
 \qquad \mathbf U_0=\mathbf X_0.
 \label{eq:ve-sde}
\end{equation}
\end{proposition}

In contrast, Proposition \ref{prop:vp} shows the VP formulation.
\begin{proposition}\label{prop:vp}
Define 
\begin{align*}
	\mathbf V_t &= \frac{\mathbf U_t}{\sqrt{1+v(t)}}, \quad \beta_{\rm VP}(t) = \frac{v'(t)}{1+v(t)}.
\end{align*}
Then,
\begin{equation}
 \dd\mathbf V_t=-\tfrac12\beta_{\rm VP}(t)\mathbf V_t\dd t
          +\sqrt{\beta_{\rm VP}(t)}\dd\widetilde{\mathbf W}_t.
 \label{eq:vp-sde}
\end{equation}
Moreover, if $\operatorname{Cov}(\mathbf X_0)=I_p$, then $\operatorname{Cov}(\mathbf V_t)=I_p$ for all $t$.
\end{proposition}
Proofs can be found in Appendix \ref{app:gen-ve} and \ref{app:gen-vp}. These identities place the residual within classical Gaussian diffusion constructions, which provide a new perspective on interpreting the residual. In contrast, the unconditioned primary process $\widehat{\mathbf X}$ generally has neither the standard VE form nor the VP variance constraint.

\section{Covariance singularity}
A main difficulty in the Markovian approximation system  \eqref{eq:general_primary_realization}--\eqref{eq:general_aux_realization} is, the conditional covariance $\sigma(t)^2$ becomes nearly singular as $t\downarrow0$. Take the one-factor case ($K=1$) as an example. If $g^2$ is continuously differentiable near zero and $g(0)>0$, the conditional covariance follows
\begin{equation}
	\sigma(t)^2=\frac{\chi^2(\lambda-\kappa_1)^2g(0)^2}{12}t^3+O(t^4).
	\label{eq:cubic}
\end{equation}
Intuitively, the primary and auxiliary processes are collinear. Computing their conditional variance subtracts nearly equal quantities, making conditioning formulas at early times unsuitable.

To avoid numerical instability, we restrict sampling to a positive interval. More precisely, we choose an early time $t_{\rm safe}$ defined as
\begin{align*}
 \min \{t \geq t_{\min} : \sigma(t) \geq \sigma_{\min}, \; \operatorname{cond}_2 C_{yy}(t) \leq M \},
\end{align*}
where $\operatorname{cond}_2$ denotes the condition number of a matrix, $M$ is a constant threshold, and $t_{\min}$ is a specified lower bound on time. Instead of proceeding to time $0$, we only consider the SDE up to time $t_{\rm safe}$, where the condition number and conditional covariance are still well-behaved.

At the lower endpoint $t_{\rm safe}$, the reverse sampling uses the clean-data denoiser instead of integrating to zero. Indeed, after completing the final Heun step, it computes the prediction of a clean image as follows:
\[
D_\theta\left(
\widehat{\mathbf X}_{t_{\rm safe}}
-\sum_{i=1}^K\eta_i(t_{\rm safe})
\mathbf Y_{t_{\rm safe}}^i,\,
t_{\rm safe}\right).
\]
This endpoint operation incurs an error that can be bounded by the initial noise levels, proved by Proposition \ref{app:gen-cutoff}.

\begin{proposition}\label{app:gen-cutoff}
	For any time $t > 0$, the Wasserstein-2 distance $W_2$ between the estimator distribution $\mathcal L(D_t^*(\boldsymbol\xi_t))$ and the initial data distribution $P_0$ is bounded by
	\begin{align}
		W_2(\mathcal L(D_t^*(\boldsymbol\xi_t)),P_0)^2 \leq p\sigma_{\mathrm{eff}}(t)^2.
		\label{app:gen-cutoff-bound}
	\end{align}
\end{proposition}

In our experiments, the time $t_{\rm safe}$ is around $0.046$ for MNIST, $0.017$ for CIFAR-10, $0.025$ for ImageNet64, and $0.043$ for ImageNet512. Hence, the final denoising approximation error is small. For example, the upper bound \eqref{app:gen-cutoff-bound} in CIFAR-10 experiments is approximately $3072(0.002368)^2\approx0.0172$.

\section{Numerical studies}
\label{sec:experiments}
We evaluate the performance on MNIST, CIFAR-10, ImageNet64, and ImageNet512. Our experiments address two practical questions: whether the augmented construction can be learned with a moderate scratch-training budget, and whether it can use an existing Gaussian denoiser. MNIST and CIFAR-10 provide evidence for the first; the ImageNet results provide evidence for the second. In particular, the frozen ImageNet512 experiment requires neither retraining the representation nor introducing auxiliary network channels.

We use the Fr\'echet inception distance (FID) to measure the generation quality. As shown in Table \ref{tab:results}, the best MNIST and CIFAR-10 FIDs are obtained with $H=0.7$, while the best FID for ImageNet experiments are obtained under $H=0.9$. It indicates that image generation favors smooth noises with $H > 1/2$. We report the main observations in this section, while the Appendix gives more details on the implementation.

\paragraph{Kernel weights and rates.}
We select the parameters of the approximate kernel $\widehat G$ in \eqref{eq:kernel} before training. For each Hurst parameter $H$, we then keep these kernel parameters fixed across datasets. Every reported configuration uses one auxiliary factor ($K=1$). For each candidate $(\lambda,\chi,\kappa_1)$, we compute the coefficient $\psi_1$ using the projection method in Proposition~\ref{prop:projection}, rather than searching for it independently. The candidates must satisfy bounds on the kernel approximation error, coefficient magnitudes, and covariance conditioning. Table~\ref{app:exp-search} gives the detailed search settings. The approximation error does not exceed $0.1$ even with one factor.

We sample rates and coefficient magnitudes on logarithmic scales. The aim is to find suitable parameters within the prescribed bounds, rather than to establish a global optimum. Table~\ref{app:exp-parameters} lists the selected parameters for $H=0.7$ and $H=0.9$. Although the kernel parameters are shared, the noise schedule and safe-time cutoff are determined separately for each dataset.

\paragraph{Evaluation.}
We use the package torch-fidelity \citep{obukhov2020torchfidelity} to calculate FID~\citep{heusel2017gans}. All FID calculations use 50,000 real and 50,000 generated images. MNIST images are converted to RGB for Inception evaluation. ImageNet preprocessing follows the corresponding pretrained donor.

\begin{table}[t]
\caption{Best FIDs. MNIST and CIFAR-10 are trained from scratch. ImageNet64 fine-tunes the main EDM2 denoiser; ImageNet512 freezes both denoisers and the decoder. Iterations count additional optimization steps for the Volterra model only and exclude donor pretraining.}
\label{tab:results}
\centering
\begin{tabular}{ccccc}
\toprule
Dataset & $H$ & Guidance & Iterations & FID $\downarrow$\\
\midrule
MNIST & $0.7$ & $1.0$ & $50$k & $0.547$\\
CIFAR-10 & $0.7$ & $1.0$ & $100$k & $1.860$\\
ImageNet64 & $0.9$ & $1.45$ & $1$k & $1.912$\\
ImageNet512 & $0.9$ & $1.85$ & $0$ & $1.903$\\
\bottomrule
\end{tabular}
\end{table}

\paragraph{Computational resources.}
The MNIST experiments use one RTX 4090 GPU with 24\,GB of memory. The other three datasets use four GPUs, each with 48\,GB of memory. We first use MNIST to evaluate our configurations, testing Hurst parameters $H\in\{0.3,0.4,\ldots,0.9\}$. For the larger datasets, we restrict the experiments to $H\in\{0.3,0.5,0.7,0.9\}$, since the best FIDs are already comparable to those reported in the literature. Our training configurations are designed for settings with relatively limited computational resources. For generation, MNIST, CIFAR-10, ImageNet64, and ImageNet512 use $200$, $50$, $100$, and $50$ Heun intervals, respectively.

\paragraph{MNIST experiments.}
We train a $6$M-parameter U-Net from scratch for $50{,}000$
iterations with a batch size of $1024$, using class conditioning
without additional guidance. Training and generation together
take approximately four hours on one RTX 4090 GPU.
Table~\ref{tab:main-mnist} reports the FIDs across Hurst
parameters. The best FID is $0.547$, obtained at $H=0.7$,
although the dependence on $H$ is non-monotone.
The best-performing configuration lies in the smoother regime
$H>1/2$, consistent with the observations of
Nobis et al.~\citep{nobis2024generative}.
Our model achieves a lower FID with a smaller network and a
shorter reported runtime. Their MNIST experiments use one
A100 GPU and take approximately $17$ hours per run, with
a model of around $15$M parameters and a best FID of $0.720$.

\begin{table}[t]
	\caption{Class-conditional MNIST FIDs across Hurst
		parameters, trained from scratch without additional
		guidance.}
	\label{tab:main-mnist}
	\centering
	\setlength{\tabcolsep}{2pt}
	\begin{tabular*}{0.8\columnwidth}
		{@{\extracolsep{\fill}}l*{7}{c}@{}}
		\toprule
		$H$ & $0.3$ & $0.4$ & $0.5$ & $0.6$
		& $0.7$ & $0.8$ & $0.9$\\
		\midrule
		FID $\downarrow$
		& $0.768$ & $0.807$ & $0.607$ & $0.786$
		& $\mathbf{0.547}$ & $0.603$ & $0.834$\\
		\bottomrule
	\end{tabular*}
\end{table}

\paragraph{CIFAR-10.}
We train a $60$M-parameter U-Net from scratch for $100{,}000$
optimization steps with a batch size of $512$, using class
conditioning without additional guidance. The reported run takes
approximately nine hours on four GPUs.
Table~\ref{tab:main-cifar} summarizes the results across Hurst
parameters. The best FID is $1.860$, obtained at $H=0.7$.
Increasing $H$ from $0.3$ to $0.7$ improves FID, while a further
increase to $0.9$ yields a slightly higher value of $1.917$.

Compared with GFDM~\citep{nobis2024generative}, which reports an
FID of $3.700$, our configuration achieves a lower FID with a
comparable model size ($60$M versus $58.7$M parameters).
GFDM uses one million optimizer updates and $128$ million image
presentations, whereas our configuration requires $100{,}000$
updates and $51.2$ million image presentations. This corresponds
to one tenth as many optimization steps and a $60\%$ reduction
in image presentations. Our FID is also close to the published conditional results of EDM ($1.790$)~\citep{karras2022elucidating} and EDM with LoRA conditioning ($1.750$)~\citep{choi2024lora}. These comparisons place our results within the range of established diffusion methods.

Empirically, we find that the linear learning-rate schedule
performs better than the OneCycle schedule in the current
setting, while increasing the batch size worsens FID.
We use a dropout rate of $0.15$, slightly higher than that of
the EDM counterpart. Generating $50{,}000$ images takes
approximately $19$ minutes on four GPUs.

\begin{table}[t]
	\caption{Class-conditional CIFAR-10 FIDs across Hurst
		parameters, trained from scratch without additional
		guidance.}
	\label{tab:main-cifar}
	\centering
	\begin{tabular*}{0.6\columnwidth}
		{@{\extracolsep{\fill}}lcccc@{}}
		\toprule
		$H$ & $0.3$ & $0.5$ & $0.7$ & $0.9$ \\
		\midrule
		FID $\downarrow$
		& $2.167$ & $1.978$ & $\mathbf{1.860}$ & $1.917$ \\
		\bottomrule
	\end{tabular*}
\end{table}

\paragraph{ImageNet64.}

A central practical feature of our framework is its compatibility
with pretrained Gaussian denoisers. We evaluate this compatibility
through short fine-tuning on ImageNet64 and entirely frozen
generation on ImageNet512. Both experiments use
autoguidance~\citep{karras2024autoguidance}, combining a main
denoiser with a smaller conditional guide:
\begin{equation}
	D^{(\omega)}
	=D_{\rm guide}
	+\omega(D_{\rm main}-D_{\rm guide}),
	\qquad \omega\geq1.
	\label{eq:guidance}
\end{equation}
Both networks receive the same class label.
Since \eqref{eq:residual-score} is affine in the denoised
prediction, the residual scores combine with the same weights,
while the analytic auxiliary score remains unchanged.
A strength of $\omega=1$ uses the main denoiser alone.

We fine-tune the native EDM2 small main network for only $1000$ optimizer updates, while keeping its smaller conditional guide frozen. The effective batch size is $4096$, corresponding to $4.096$ million image presentations. The pretrained architecture and denoiser interface are preserved; the auxiliary state does not introduce additional network input channels. We restrict the effective-noise interval to $[0.007,80]$ and use $100$ stochastic Heun intervals for generation. Generating $50{,}000$ images takes approximately $2$ hours on four GPUs. The best recorded FID is $1.912$, obtained at $H=0.9$ and $\omega=1.45$. 

Table~\ref{app:exp-imagenet64-guidance} shows that the effect of
$H$ depends on the guidance strength.
The settings $H=0.5$, $0.7$, and $0.9$ all attain their best
recorded FIDs at $\omega=1.45$, whereas $H=0.3$ requires the
larger strength $\omega=1.6$.
The differences among the larger Hurst settings become smaller
near their guidance optimum: the FID gap between $H=0.9$ and
$H=0.5$ decreases from approximately $0.042$ at $\omega=1$ to
$0.012$ at $\omega=1.45$.
Nevertheless, the ranking is not uniform across guidance
strengths. Although $H=0.9$ gives the overall minimum,
$H=0.7$ gives lower recorded FIDs at $\omega=1.6$ and $\omega=2$.
The best values for $H=0.9$ and $H=0.7$ differ by only
approximately $0.004$.

\paragraph{ImageNet512.}
We introduce the Volterra augmentation into the donor's $4\times64\times64$ latent space and use a frozen VAE decoder to produce $512\times512$ RGB images. The native EDM2 small main network, its smaller conditional guide, and the decoder all remain frozen. Consequently, this experiment requires no additional gradient updates. This provides a configuration for evaluating different Volterra kernels without retraining a high-resolution generative model. Generation uses an effective-noise interval of $[0.02,80]$ and $50$ stochastic Heun intervals. It takes about 1 hour and 40 minutes to generate $50{,}000$ images on four GPUs. Figure \ref{fig:imagenet512} shows our generated ImageNet512 images with the best FID.

The interaction with guidance differs from ImageNet64. As shown in Table~\ref{app:exp-imagenet512-guidance}, $H=0.9$ gives the lowest recorded FID at every sampled guidance strength, including strengths near the empirical optimum. Its benefit is therefore not confined to insufficient or excessive guidance. After selecting guidance separately for each Hurst setting, the best FIDs are $1.903$, $1.929$, $1.946$, and $1.995$ for $H=0.9$, $0.7$, $0.5$, and $0.3$, respectively. The corresponding optimal recorded strengths are $1.85$, $1.9$, $1.9$, and $2$. Thus, the preferred guidance strength varies with $H$. Autoguidance remains important even for the best Hurst setting: at $H=0.9$, FID decreases from $3.787$ without additional guidance to $1.903$ at $\omega=1.85$, then increases to $2.461$ at $\omega=2.6$.

These observations support joint selection of the Hurst parameter and guidance strength, rather than a universal rule that increasing $H$ improves generation or compensates for poorly chosen guidance.

ImageNet seems to favor more on the smooth region ($H >1/2$) than MNIST and CIFAR-10, since the best FIDs are obtained at $H=0.9$. A possible explanation concerns the temporal weighting of the perturbations. Fractional kernels with $H>1/2$ reduce the contribution of the most recent Brownian increments and produces more gradual temporal variation. In a finite-step implementation, this different weighting may help the denoising iterates evolve more consistently. Such behavior could be useful for complex images, where global object structure and fine details must remain consistent throughout generation. Nevertheless, this is a heuristic explanation rather than a demonstrated mechanism.

\begin{table}[t]
	\caption{ImageNet64 FIDs with $100$ Heun intervals.
		Bold entries mark the minimum in each Hurst column.}
	\label{app:exp-imagenet64-guidance}
	\centering
	\begin{tabular}{@{}rrrrr@{}}
		\toprule
		Guidance & $H=0.9$ & $H=0.7$ & $H=0.5$ & $H=0.3$\\
		\midrule
		$1.0$    & $2.319$ & $2.331$ & $2.361$ & $3.199$\\
		$1.3$  & $1.966$ & $1.969$ & $1.978$ & $2.694$\\
		$1.35$ & $1.935$ & $1.939$ & $1.951$ & $2.634$\\
		$1.4$  & $1.921$ & $1.923$ & $1.937$ & $2.574$\\
		$1.45$ & $\mathbf{1.912}$ & $\mathbf{1.916}$
		& $\mathbf{1.924}$ & $2.523$\\
		$1.5$  & $1.918$ & $1.918$ & $1.925$ & $2.493$\\
		$1.6$  & $1.953$ & $1.948$ & $1.953$
		& $\mathbf{2.435}$\\
		$2.0$    & $2.366$ & $2.362$ & $2.374$ & $2.603$\\
		\bottomrule
	\end{tabular}
\end{table}

\begin{table}[t]
	\setlength{\belowcaptionskip}{6pt}
	\caption{ImageNet512 FIDs with frozen EDM2 networks
		and $50$ Heun intervals. Bold entries mark the minimum in
		each Hurst column.}
	\label{app:exp-imagenet512-guidance}
	\centering
	\begin{tabular}{@{}rrrrr@{}}
		\toprule
		Guidance & $H=0.9$ & $H=0.7$ & $H=0.5$ & $H=0.3$\\
		\midrule
		$1.0$    & $3.787$ & $3.882$ & $3.881$ & $4.231$\\
		$1.4$  & $2.367$ & $2.422$ & $2.398$ & $2.566$\\
		$1.8$  & $1.916$ & $1.948$ & $1.957$ & $2.029$\\
		$1.85$ & $\mathbf{1.903}$ & $1.935$ & $1.956$ & $2.008$\\
		$1.9$  & $1.914$ & $\mathbf{1.929}$
		& $\mathbf{1.946}$ & $2.013$\\
		$1.95$ & $1.915$ & $1.947$ & $1.956$ & $1.999$\\
		$2.0$    & $1.930$ & $1.947$ & $1.964$
		& $\mathbf{1.995}$\\
		$2.1$  & $1.973$ & $1.987$ & $2.006$ & $2.029$\\
		$2.2$  & $2.030$ & $2.058$ & $2.072$ & $2.089$\\
		$2.6$  & $2.461$ & $2.492$ & $2.483$ & $2.480$\\
		\bottomrule
	\end{tabular}
\end{table}

\begin{figure}[t]
	\centering
	\includegraphics[width=0.49\linewidth]{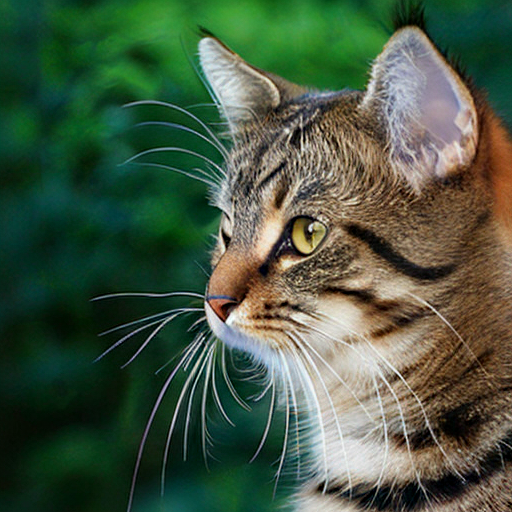}
	\hfill
	\includegraphics[width=0.49\linewidth]{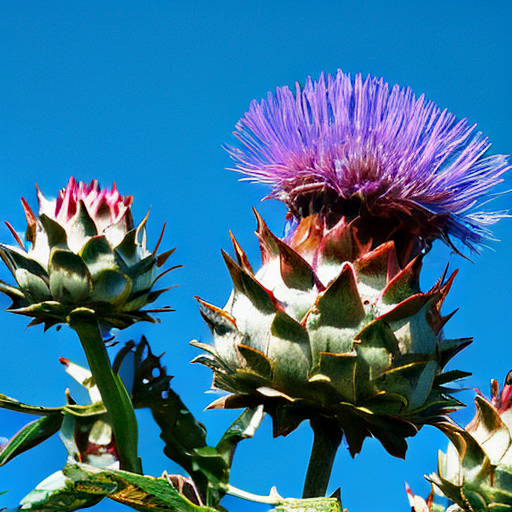}
	\\[2pt]
	\includegraphics[width=0.49\linewidth]{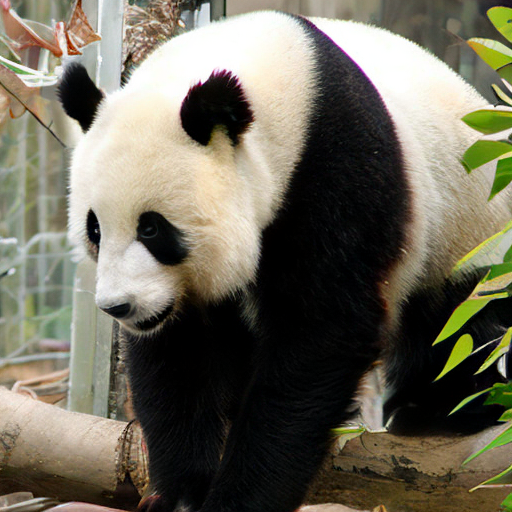}
	\hfill
	\includegraphics[width=0.49\linewidth]{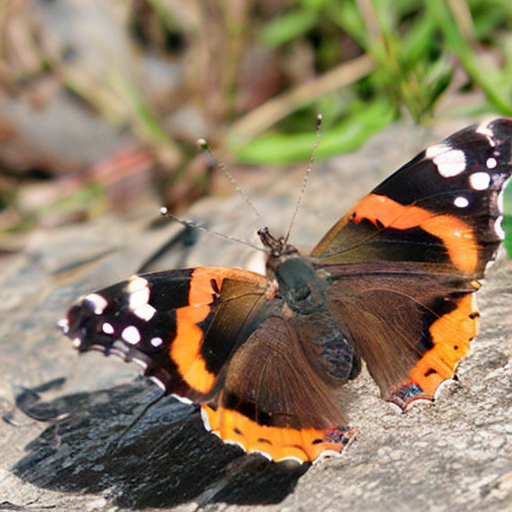}
	\caption{Generated ImageNet512 images under $H=0.9$, guidance $=1.85$, with the best FID $=1.903$. More images are in the Appendix.}
	\label{fig:imagenet512}
\end{figure}

\paragraph{Comparisons.}

\begin{table}[t]
	\setlength{\belowcaptionskip}{6pt}
	\caption{Class-conditional ImageNet512 FIDs. CFG denotes classifier-free guidance, AG autoguidance, and ADM-G classifier guidance. The SADM row uses the DiT-SADM result in the published table. Simple diffusion reports FID against its training reference for the listed configuration. The sCD-XXL row is a distilled two-step consistency model. Classifier evaluations and VAE decoding are excluded from NFE.}
	\label{tab:literature-imagenet512}
	\centering
	\begin{tabularx}{0.8\columnwidth}{@{}Xrr@{}}
		\toprule
		Method & FID $\downarrow$ & NFE\\
		\midrule
		ADM-G~\cite{dhariwal2021diffusion}
		& $7.720$ & $250+250$\\
		U-ViT-H/4 with CFG~\cite{bao2023all}
		& $4.050$ & --\\
		DiT-XL/2 with CFG~\cite{peebles2023scalable}
		& $3.040$ & $250+250$\\
		Simple diffusion~\cite{hoogeboom2023simple}
		& $3.020$ & --\\
		CAN, CaT-L1~\cite{cai2024condition}
		& $2.480$ & --\\
		VDM++ with CFG~\cite{kingma2023understanding}
		& $2.650$ & --\\
		DiT-SADM~\cite{yang2024structure}
		& $2.110$ & --\\
		EDM2-S with CFG~\cite{karras2024analyzing}
		& $2.230$ & $63+63$\\
		EDM2-XXL with CFG~\cite{karras2024analyzing}
		& $1.810$ & $63+63$\\
		EDM2-XXL, limited CFG~\cite{kynkaanniemi2024limited}
		& $1.400$ & --\\
		EDM2-S with AG~\cite{karras2024autoguidance}
		& $1.340$ & $63+63$\\
		EDM2-XXL with AG~\cite{karras2024autoguidance}
		& $1.250$ & $63+63$\\
		sCD-XXL, two steps~\cite{lu2025simplifying}
		& $1.880$ & $2$\\
		\midrule
		Ours, frozen EDM2-S with AG
		& $1.903$ & $101+101$\\
		\bottomrule
	\end{tabularx}
\end{table}

Table~\ref{tab:literature-imagenet512} places our FID of $1.903$ in the context of published results. The frozen ImageNet512 configuration reports a lower FID than several models, including U-ViT, DiT, simple diffusion, CAN, VDM++, and DiT-SADM. The original EDM2 work and the two-step sCD-XXL model report lower values.  Our experiments establish that pretrained denoisers can support the Volterra construction with little additional training; they do not guarantee an improvement over the original samplers automatically.

Training from scratch therefore remains an important part of the evaluation. It allows the network to be optimized using the noise-level distribution and loss weighting selected for our implementation. Although the population Gaussian denoising problem is unchanged after normalization, these choices affect how a finite network distributes its approximation error across noise levels. The scratch-training experiments and pretrained experiments thus address complementary questions: whether the model can be learned with a limited training budget, and whether an existing denoiser can be reused without substantial retraining.

For ImageNet512, freezing the networks was also motivated by our preliminary fine-tuning experiments. Under the settings we tested, fine-tuning increased FID relative to retaining the pretrained weights. A short fine-tuning stage may alter a well-trained denoiser without providing enough optimization to reach a better solution. We therefore used the frozen main network and guide in the reported experiments. This observation highlights the value of large-scale pretraining and careful optimization in \cite{karras2024analyzing}.

\section{Conclusion}
\label{sec:conclusion}

We developed a Volterra formulation of generative modeling with finite-dimensional Markovian approximations. Approximating the convolution kernel gives direct control of the forward mean-square error, while Gaussian conditioning reduces the learned score to a residual denoising problem in the original data dimension. The pretrained weights can be used in Volterra generative models. Experiments support selecting the Hurst parameter and guidance strength jointly. 

Future work can study a joint bound on kernel approximation, score estimation, and numerical integration errors. Another direction is to select kernel parameters and time grids jointly using covariance conditioning. Larger-scale training from scratch, would further clarify how the kernel choice interacts with denoiser training and guidance.

\label{main-content-end}

\section*{LLM Usage Statement}
ChatGPT, including Codex, was used to assist with development and revision of the Python implementation, mathematical proofs, literature organization, manuscript writing, editing, and \LaTeX{} preparation. The numerical results are measurements from the reported experiments, not values generated by a language model. The authors are responsible for the final scientific claims, citations, code, and experimental results.
\appendix
\numberwithin{equation}{section}
\numberwithin{theorem}{section}
\numberwithin{table}{section}
\raggedbottom
\section{Approximation and covariance}
\label{app:proj}

Throughout this section, $T>0$, $G\in L^2(0,T)$ is
real-valued, $\lambda>0$, and the auxiliary rates
$\kappa_1,\ldots,\kappa_K$ are nonnegative and distinct.
The covariance results below additionally require
$\chi\ne0$ and
$\lambda\notin\{\kappa_1,\ldots,\kappa_K\}$.
These additional conditions are unnecessary for the
fixed-rate projection.
Both the state and the Brownian motion have dimension $p$.

\subsection{Proof of Proposition~\ref{prop:kernel-process}}
\label{app:proj:densityerror}
\label{app:proj:processerror}

\begin{proof}
	Let $\mathbf b:[0,T]\to\mathbb R^p$ and
	$\boldsymbol\Sigma:[0,T]\to\mathbb R^{p\times p}$
	be bounded, measurable, deterministic functions, as assumed
	in the main text. Couple $\mathbf X$ and
	$\widehat{\mathbf X}$ using the same initial state and
	Brownian motion, and write $\Delta G=G-\widehat G$.
	The initial-state terms cancel, giving
	\begin{align*}
		\mathbf X_t-\widehat{\mathbf X}_t
		&=\int_0^t\Delta G(t-s)\mathbf b(s)\dd s +
		\int_0^t\Delta G(t-s)\boldsymbol\Sigma(s)\dd\mathbf B_s.
	\end{align*}
	Consequently, It\^o's isometry yields
	\begin{align}
		\mathbb E\|\mathbf X_t-\widehat{\mathbf X}_t\|^2
		&=
		\left\|\int_0^t\Delta G(t-s)\mathbf b(s)\dd s\right\|^2
		 +
		\int_0^t|\Delta G(t-s)|^2
		\|\boldsymbol\Sigma(s)\|_{\mathrm{HS}}^2\dd s. 	\label{app:proj:isometry} 
	\end{align}
	By the Cauchy--Schwarz inequality,
	\begin{align*}
		 \left\|\int_0^t\Delta G(t-s)\mathbf b(s)\dd s\right\|^2 &\le
		\left(\int_0^t|\Delta G(t-s)|^2\dd s\right) \left(\int_0^t\|\mathbf b(s)\|^2\dd s\right) \\
		&\le
		t\|\mathbf b\|_\infty^2 \int_0^t|\Delta G(u)|^2\dd u.
	\end{align*}
	The second term in \eqref{app:proj:isometry} is bounded by $\|\boldsymbol\Sigma\|_{\infty,\mathrm{HS}}^2
	\int_0^t|\Delta G(u)|^2\dd u$.

	Taking the supremum over $t\in[0,T]$ therefore gives
	\begin{equation*}
		\begin{aligned}
			&\sup_{0\le t\le T}
			\mathbb E\|\mathbf X_t-\widehat{\mathbf X}_t\|^2 \le
			\left(T\|\mathbf b\|_\infty^2
			+\|\boldsymbol\Sigma\|_{\infty,\mathrm{HS}}^2\right)
			\|G-\widehat G\|_{L^2(0,T)}^2.
		\end{aligned}
	\end{equation*}
	Here, the supremum norms are essential suprema over time.
\end{proof}

\subsection{Proof of Proposition~\ref{prop:projection}}
\label{app:proj:projection}
\label{app:proj:ridge}

\begin{proof}
	Write $h(u)=G(u)-\chi e^{-\lambda u}$ and define
	\begin{equation}
		\begin{aligned}
			A_{ij}
			&=\int_0^T e^{-(\kappa_i+\kappa_j)u}\dd u, \quad r_i
			=\int_0^T h(u)e^{-\kappa_i u}\dd u.
		\end{aligned}
		\label{app:proj:entries}
	\end{equation}
	For every $\widetilde\psi\in\mathbb R^K$,
	\begin{equation*}
		\widetilde\psi^\top A\widetilde\psi
		=
		\int_0^T
		\left(\sum_{i=1}^K\widetilde\psi_i e^{-\kappa_i u}\right)^2
		\dd u.
	\end{equation*}
	If this integral is zero, continuity implies that the
	exponential sum vanishes on $[0,T]$. Differentiating at zero
	gives
	\begin{equation}
		\sum_{i=1}^K(-\kappa_i)^j\widetilde\psi_i=0,
		\qquad j=0,\ldots,K-1.
		\label{app:proj:vandermonde}
	\end{equation}
	The corresponding Vandermonde matrix is invertible because
	the auxiliary rates are distinct. Hence
	$\widetilde\psi=0$, proving $A\succ0$.
	
	For $\delta\ge0$, the penalized objective is
	\begin{align}
		J_\delta(\widetilde\psi)
		&:=
		\left\|h-\sum_{i=1}^K
		\widetilde\psi_i e^{-\kappa_i\cdot}\right\|_{L^2(0,T)}^2
		+\delta\|\widetilde\psi\|^2 =
		\|h\|_{L^2(0,T)}^2
		-2\widetilde\psi^\top r
		+\widetilde\psi^\top(A+\delta I_K)\widetilde\psi.
		\label{app:proj:objective}
	\end{align}
	Its Hessian is $2(A+\delta I_K)\succ0$, so its unique
	minimizer satisfies
	\begin{equation*}
		(A+\delta I_K)\psi=r,
		\qquad
		\psi=(A+\delta I_K)^{-1}r.
	\end{equation*}
	At this minimizer, the unpenalized squared error is
	\begin{align*}
		\|G-\widehat G\|_{L^2(0,T)}^2
		&=
		\|h\|_{L^2(0,T)}^2
		-2\psi^\top r+\psi^\top A\psi =
		\|h\|_{L^2(0,T)}^2
		-\psi^\top r-\delta\|\psi\|^2.
	\end{align*}
	The first equality is the quadratic expansion of the kernel
	error; the second follows from
	$\psi^\top A\psi=\psi^\top r-\delta\|\psi\|^2$.
\end{proof}

For the fractional kernel
$G_H(u)=u^{H-1/2}/\Gamma(H+1/2)$ with $H\in(0,1)$,
the projection quantities have explicit expressions.
First,
\begin{equation*}
	\|G_H\|_{L^2(0,T)}^2
	=
	\frac{T^{2H}}{2H\Gamma(H+1/2)^2}.
\end{equation*}
For a scalar rate $\nu\ge0$, define
\begin{align*}
	L(\nu)
	&:=
	\int_0^T e^{-\nu u}\dd u
	=
	\begin{cases}
		\dfrac{1-e^{-\nu T}}{\nu},&\nu>0,\\[4pt]
		T,&\nu=0,
	\end{cases} \\
	B_H(\nu)
	&:=
	\int_0^T G_H(u)e^{-\nu u}\dd u  
	=
	\begin{cases}
		\dfrac{\gamma(H+\frac12,\nu T)}
		{\Gamma(H+\frac12)\nu^{H+1/2}},
		&\nu>0,\\[6pt]
		\dfrac{T^{H+1/2}}{\Gamma(H+\frac32)},
		&\nu=0.
	\end{cases}
\end{align*}
Here,
$\gamma(a,x)=\int_0^x z^{a-1}e^{-z}\dd z$
is the lower incomplete gamma function.
The entries in \eqref{app:proj:entries} become
\[
A_{ij}=L(\kappa_i+\kappa_j),
\qquad
r_i=B_H(\kappa_i)-\chi L(\lambda+\kappa_i),
\]
and
\begin{equation*}
		\|h\|_{L^2(0,T)}^2
		=
		\|G_H\|_{L^2(0,T)}^2
		-2\chi B_H(\lambda)
		+\chi^2L(2\lambda).
\end{equation*}

\paragraph{Approximation by a growing exponential family.}
\label{app:proj:density}
The existence of arbitrarily accurate exponential
approximations follows from the Stone--Weierstrass theorem
\cite[Theorem~7.32]{rudin1976principles}.
Choose $\rho>\lambda$ and consider the real algebra
\begin{equation*}
	\operatorname{span}
	\{u\mapsto e^{-n\rho u}:n=1,2,\ldots\}.
\end{equation*}
It is closed under multiplication and contains the strictly
positive, strictly decreasing function $e^{-\rho u}$.
It therefore separates points and vanishes at no point of
$[0,T]$. The theorem implies uniform density in $C([0,T])$.
Since continuous functions are dense in $L^2(0,T)$, this
algebra is also dense in $L^2(0,T)$. In particular, for any
fixed $\chi\in\mathbb R$,
\begin{equation*}
	\lim_{K\to\infty}
	\min_{\psi\in\mathbb R^K}
	\left\|
	G-\chi e^{-\lambda\cdot}
	-\sum_{i=1}^K\psi_i e^{-i\rho\cdot}
	\right\|_{L^2(0,T)}^2
	=0.
\end{equation*}
The auxiliary rates $\kappa_i=\rho i$ are distinct and differ
from $\lambda$.

%%%%%%%%%%%%%%%%%%%

\subsection{Projected dynamics}
\label{app:proj:realization}
\label{app:proj:lift}

Suppose $\mathbf b\in L^1(0,T;\mathbb R^p)$ and $\boldsymbol\Sigma\in L^2(0,T;\mathbb R^{p\times p})$ are deterministic. The auxiliary equations have the explicit solutions
\begin{align*}
	\mathbf Y_t^i
	&=\int_0^t e^{-\kappa_i(t-s)}\mathbf b(s)\dd s +
	\int_0^t e^{-\kappa_i(t-s)}
	\boldsymbol\Sigma(s)\dd\mathbf B_s.
\end{align*}
Subtracting the weighted auxiliary equations from the
primary equation and using
$\varepsilon-\sum_i\psi_i=\chi$ gives
\begin{align*}
	\dd\left(\widehat{\mathbf X}_t-\sum_i\psi_i\mathbf Y_t^i\right)
	&=-\lambda
	\left(\widehat{\mathbf X}_t-\sum_i\psi_i\mathbf Y_t^i\right)\dd t + \chi\mathbf b(t)\dd t
	+\chi\boldsymbol\Sigma(t)\dd\mathbf B_t.
\end{align*}
The initial value of this difference is $\mathbf X_0$. Then
\begin{align}
	\widehat{\mathbf X}_t
	&=\alpha(t)\mathbf X_0
	+\int_0^t\widehat G(t-s)\mathbf b(s)\dd s +
	\int_0^t\widehat G(t-s)
	\boldsymbol\Sigma(s)\dd\mathbf B_s.
	\label{app:proj:exactconvolution}
\end{align}
All components are driven by the same Brownian motion.

\subsection{Proof of Proposition~\ref{prop:covariance}}
\label{app:proj:covariance}
\label{app:proj:positive}

\begin{proof}
For each data coordinate, the innovation covariance $C(t)$ can be written as
\begin{equation*}
		\begin{aligned}
			C(t)&=\int_0^t g(s)^2k(t-s)k(t-s)^\top\dd s, \quad k(u)= \bigl(\widehat G(u), e^{-\kappa_1u}, \ldots, e^{-\kappa_Ku}\bigr)^\top.
		\end{aligned}
\end{equation*}

For $\ell=(\ell_0,\ldots,\ell_K)^\top\in\mathbb R^{K+1}$, the quadratic form $\ell^\top C(t)\ell$ equals
	\begin{equation*}
		\int_0^t g(s)^2
		\left[
		\chi\ell_0e^{-\lambda(t-s)}
		+\sum_{i=1}^K(\ell_0\psi_i+\ell_i)
		e^{-\kappa_i(t-s)}
		\right]^2\dd s.
	\end{equation*}
By the same Vandermonde argument used in \eqref{app:proj:vandermonde},
	\[
	\chi\ell_0=0,
	\qquad
	\ell_0\psi_i+\ell_i=0,\quad i=1,\ldots,K.
	\]
	Since $\chi\ne0$, these identities imply $\ell=0$. Therefore, for every $t>0$, $C(t)\succ0$. Its principal block $C_{yy}(t)$ is positive definite, and its Schur complement satisfies
	\begin{equation*}
		\begin{aligned}
			\sigma(t)^2
			&=c_{x\mid y}(t) = c_{xx}(t)
			-c_{yx}(t)^\top C_{yy}(t)^{-1}c_{yx}(t)>0.
		\end{aligned}
	\end{equation*}
	
	The Schur complement is the minimum variance obtained
	by subtracting a linear combination of the auxiliary
	innovations from the primary innovation. Therefore,
	\begin{align}
	 \sigma(t)^2   &=\min_{a\in\mathbb R^K}\int_0^t g(s)^2
		\left(\widehat G(t-s)
		-\sum_{i=1}^K a_i e^{-\kappa_i(t-s)}\right)^2
		\dd s \label{app:proj:residualprojection} \\
		&=\chi^2\min_{a\in\mathbb R^K}\int_0^t g(s)^2
		\left(e^{-\lambda(t-s)}
		-\sum_{i=1}^K a_i e^{-\kappa_i(t-s)}\right)^2
		\dd s. \notag
	\end{align}
	The second equality follows by writing the coefficients
	in the first minimum as $\psi+\chi a$.
	Dividing by $\alpha(t)^2$ and rescaling the free
	coefficients gives
	\begin{equation*}
		\sigma_{\mathrm{eff}}(t)^2
		=\chi^2\min_{a\in\mathbb R^K}
		\int_0^t g(s)^2
		\left(e^{\lambda s}
		-\sum_{i=1}^K a_i e^{\kappa_i s}\right)^2\dd s.
	\end{equation*}
	
	The minimizer exists and is unique because the auxiliary
	exponentials have a positive definite weighted Gram matrix
	on every interval $(0,t)$.
	Fix $0<t_1<t_2\le T$, and let $a_*$ minimize the integral
	at $t_2$. Splitting that integral at $t_1$ yields
	\begin{align*}
		&\sigma_{\mathrm{eff}}(t_2)^2
		-\sigma_{\mathrm{eff}}(t_1)^2 \ge
		\chi^2\int_{t_1}^{t_2}g(s)^2
		\left(e^{\lambda s}
		-\sum_{i=1}^K(a_*)_i e^{\kappa_i s}\right)^2
		\dd s>0.
	\end{align*}
	The final inequality follows from positivity of $g$
	and linear independence of exponentials with distinct
	rates. Thus $\sigma_{\mathrm{eff}}$ is strictly increasing.
	Continuity on $(0,T]$ follows from the covariance integrals
	and continuity of inversion on positive definite matrices.
	Finally, choosing $a=0$ gives
	\[
	0\le\sigma_{\mathrm{eff}}(t)^2
	\le\chi^2\int_0^t g(s)^2e^{2\lambda s}\dd s
	\longrightarrow0
	\qquad(t\downarrow0).
	\]
	This proves the asserted continuous extension at zero.
\end{proof}

\section{Residual denoising}
\label{app:gen-theory}

\subsection{Gaussian residual and denoising objective}
\label{app:gen-residual-section}

\begin{proposition}[Gaussian residual]
	\label{app:gen-residual}
	For every $t>0$, the residual $\boldsymbol\xi_t$ defined in the main text
	has the representation
	\begin{equation}
		\boldsymbol\xi_t
		=\alpha(t)\mathbf X_0+\sigma(t)\boldsymbol\zeta_t,
		\qquad
		\boldsymbol\zeta_t\sim\mathcal N(0,I_p),
		\label{app:gen-residual-law}
	\end{equation}
	where $\boldsymbol\zeta_t$ is independent of
	$(\mathbf X_0,\mathbf Y_t)$. Consequently,
	$\boldsymbol\xi_t$ and $\mathbf Y_t$ are independent
	at each fixed time.
\end{proposition}

\begin{proof}
	Denote
	\[
	\mathbf N_t=\widehat{\mathbf X}_t-\alpha(t)\mathbf X_0,
	\qquad
	\mathbf E_t=\mathbf N_t-\sum_{i=1}^K\eta_i(t)\mathbf Y_t^i.
	\]
	The pair $(\mathbf N_t,\mathbf Y_t)$ is jointly Gaussian
	and independent of $\mathbf X_0$. The covariance blocks
	defined in the main text give
	\begin{align*}
		\operatorname{Cov}(\mathbf Y_t,\mathbf E_t)
		&=(c_{yx}-C_{yy}\eta)\otimes I_p=0,\\
		\operatorname{Cov}(\mathbf E_t)
		&=\left(c_{xx}-2\eta^\top c_{yx}
		+\eta^\top C_{yy}\eta\right)I_p
		=\sigma(t)^2I_p.
	\end{align*}
	Joint Gaussianity makes $\mathbf E_t$ independent of
	$\mathbf Y_t$. Since the pair
	$(\mathbf E_t,\mathbf Y_t)$ is independent of
	$\mathbf X_0$, the variable
	$\boldsymbol\zeta_t=\mathbf E_t/\sigma(t)$ is
	independent of $(\mathbf X_0,\mathbf Y_t)$.
	Now $\boldsymbol\xi_t=\alpha(t)\mathbf X_0+\mathbf E_t$,
	which proves the result.
\end{proof}

\begin{proposition}[Population denoising objective]
	\label{app:gen-bayes}
	Suppose the loss in \eqref{eq:loss} is finite for at
	least one measurable denoiser. Among denoisers with finite
	loss, its minimizer is
	$D_t^*(\boldsymbol\xi)
	=\mathbb E[\mathbf X_0\mid\boldsymbol\xi_t
	=\boldsymbol\xi]$, up to null sets under the sampled
	time and residual distribution.
\end{proposition}

\begin{proof}
	At a fixed positive time, conditional expectation and
	orthogonality give
	\begin{align*}
		&\mathbb E\|\mathcal D(\boldsymbol\xi_t,t)
		-\mathbf X_0\|^2 =
		\mathbb E\|D_t^*(\boldsymbol\xi_t)-\mathbf X_0\|^2
		+\mathbb E\|\mathcal D(\boldsymbol\xi_t,t)
		-D_t^*(\boldsymbol\xi_t)\|^2.
	\end{align*}
	Integrating this identity over the sampled times with
	weight $w(t)>0$ proves optimality. Equality of losses
	forces the final term to vanish at almost every sampled
	time.
\end{proof}

\subsection{Proof of Proposition~\ref{prop:transfer}}
\label{app:gen-transfer-section}

\begin{proof}
	By Proposition~\ref{app:gen-residual},
	\[
	\frac{\boldsymbol\xi_t}{\alpha(t)}
	=\mathbf X_0+\sigma_{\mathrm{eff}}(t)
	\boldsymbol\zeta_t.
	\]
	The conditional distribution of $\mathbf X_0$ given $\boldsymbol\xi_t=\boldsymbol\xi$ equals its conditional distribution given $\boldsymbol\xi/\alpha(t)$. This proves \eqref{eq:population-transfer}.
	
	For the network identity, set
	$s=\sigma_{\mathrm{eff}}(t)$ and
	$\mathbf u=\boldsymbol\xi/\alpha(t)$. Since
	\[
	m(t)=\alpha(t)
	\sqrt{s^2+\sigma_{\mathrm{data}}^2},
	\]
	the coefficients in \eqref{eq:precondition} satisfy
	\begin{align*}
		c_{\mathrm{in}}(t)\boldsymbol\xi
		&=\frac{\mathbf u}
		{\sqrt{s^2+\sigma_{\mathrm{data}}^2}}, \quad 
		c_{\mathrm{skip}}(t)\boldsymbol\xi
		=\frac{\sigma_{\mathrm{data}}^2}
		{s^2+\sigma_{\mathrm{data}}^2}\mathbf u, \quad 
		c_{\mathrm{out}}(t)
		=\frac{s\sigma_{\mathrm{data}}}
		{\sqrt{s^2+\sigma_{\mathrm{data}}^2}}, \quad 
		c_{\mathrm{noise}}(t) =\tfrac14\log s.
	\end{align*}
	These are the input, skip, output, and noise-embedding
	coefficients of the EDM denoiser evaluated at
	$(\mathbf u,s)$. The equality in
	\eqref{eq:network-transfer} follows for every choice
	of backbone weights.
\end{proof}

\subsection{Proof of Proposition~
	\ref{prop:score_decomposition_unanchored_main}}
\label{app:gen-joint-score}

\begin{proof}
	Proposition~\ref{app:gen-residual} gives independence
	of $\boldsymbol\xi_t$ and $\mathbf Y_t$. The change of
	variables
	\[
	(\mathbf x,\mathbf y)\longmapsto
	\left(\mathbf x-\sum_{i=1}^K\eta_i(t)\mathbf y^i,
	\mathbf y\right)
	\]
	has Jacobian determinant one. This proves the density
	factorization in \eqref{eq:factorization}.
	Writing
	$\boldsymbol\xi=\mathbf x-\sum_i\eta_i(t)\mathbf y^i$
	and differentiating its logarithm gives
	\begin{align*}
		\nabla_{\mathbf x}\log p_t(\mathbf x,\mathbf y)
		&=\mathbf s^*(\boldsymbol\xi,t),\\
		\nabla_{\mathbf y}\log p_t(\mathbf x,\mathbf y)
		&=-(\eta(t)\otimes I_p)\mathbf s^*(\boldsymbol\xi,t) -(C_{yy}(t)^{-1}\otimes I_p)\mathbf y.
	\end{align*}
	Here the final term is the score of
	$q_t^Y=\mathcal N(0,C_{yy}(t)\otimes I_p)$.
	These identities give
	\eqref{eq:score_decomp_unanchored_main}.
	
	It remains to justify the residual score used there.
	Let $\varphi_{\sigma^2I_p}$ denote the centered Gaussian
	density. Equation~\eqref{app:gen-residual-law} implies
	\begin{align*}
%		p_t^\xi(\boldsymbol\xi)
%		&=\int\varphi_{\sigma(t)^2I_p}
%		\bigl(\boldsymbol\xi-\alpha(t)\mathbf x_0\bigr)
%		P_0(\dd\mathbf x_0),\\
		\nabla_{\boldsymbol\xi}p_t^\xi(\boldsymbol\xi)
		&=\int
		\frac{\alpha(t)\mathbf x_0-\boldsymbol\xi}
		{\sigma(t)^2}
		\varphi_{\sigma(t)^2I_p}
		\bigl(\boldsymbol\xi-\alpha(t)\mathbf x_0\bigr)
		P_0(\dd\mathbf x_0).
	\end{align*}
	Divide by $p_t^\xi(\boldsymbol\xi)$. Then Bayes' theorem yields
	\[
	\mathbf s^*(\boldsymbol\xi,t)
	=\frac{\alpha(t)D_t^*(\boldsymbol\xi)
		-\boldsymbol\xi}{\sigma(t)^2},
	\]
	as stated in \eqref{eq:residual-score}.
\end{proof}

\subsection{Proof of Proposition~\ref{prop:ve}}
\label{app:gen-ve}

\begin{proof}
	Define
	\begin{align*}
		\mathbf M_t^0
		&=\int_0^t e^{\lambda s}g(s)\dd\mathbf B_s, \quad
		\mathbf M_t^i=\int_0^t e^{\kappa_i s}g(s)\dd\mathbf B_s.
	\end{align*}
	Stack $\mathbf M_t^1,\ldots,\mathbf M_t^K$ into
	$\mathbf M_t$ as column vectors. \eqref{app:proj:exactconvolution} gives
	\begin{align*}
		\mathbf Y_t^i&=e^{-\kappa_i t}\mathbf M_t^i, \qquad
		\widehat{\mathbf X}_t
		=\alpha(t)\mathbf X_0
		+\chi\alpha(t)\mathbf M_t^0
		+\sum_{i=1}^K\psi_i e^{-\kappa_i t}\mathbf M_t^i.
	\end{align*}
	For $t>0$, define
	\begin{align*}
		f(s)&=(e^{\kappa_1s},\ldots,e^{\kappa_Ks})^\top, \quad V(t) =\int_0^t g(s)^2f(s)f(s)^\top\dd s, \quad c(t)=\int_0^t g(s)^2f(s)e^{\lambda s}\dd s,\\
		\rho(t)&=\int_0^t g(s)^2e^{2\lambda s}\dd s,
		\qquad a(t)=V(t)^{-1}c(t).
	\end{align*}
	The matrix $V(t)$ is positive definite by linear
	independence of the auxiliary exponentials. Thus
	\(\sum_i a_i(t)\mathbf M_t^i\) is the linear projection of \(\mathbf M_t^0\) onto $\mathbf M_t$. Since
	$\mathbf Y_t^i=e^{-\kappa_i t}\mathbf M_t^i$, the
	conditioning coefficient satisfies
	\[
	\eta_i(t)
	=\psi_i+\chi\alpha(t)e^{\kappa_i t}a_i(t).
	\]
	Substituting into the definition of
	$\boldsymbol\xi_t$ gives
	\begin{equation*}
		\mathbf U_t=\mathbf X_0+\mathbf R_t,
		\qquad
		\mathbf R_t
		=\chi\left(\mathbf M_t^0
		-(a(t)^\top\otimes I_p)\mathbf M_t\right).
	\end{equation*}
	In particular, $\mathbf R$ is centered Gaussian and
	independent of $\mathbf X_0$. For $j=1,\ldots,p$, the coordinate variance is
	\begin{equation*}
		\operatorname{Var}(R_t^{(j)})
		=\chi^2\bigl(\rho(t)-c(t)^\top V(t)^{-1}c(t)\bigr)
		=v(t).
	\end{equation*}
	
	For $0<s\le t$, independent Brownian increments and
	$V(s)a(s)=c(s)$ give
	\begin{align*}
		\operatorname{Cov}(\mathbf R_s,\mathbf M_t)
		&=\chi\bigl(c(s)^\top-a(s)^\top V(s)\bigr)
		\otimes I_p=0,\\
		\operatorname{Cov}(\mathbf R_s,\mathbf R_t)
		&=\chi\,\operatorname{Cov}
		(\mathbf R_s,\mathbf M_t^0) = \chi^2\bigl(\rho(s)-a(s)^\top c(s)\bigr)I_p
		=v(s)I_p.
	\end{align*}
	Thus the covariance of $\mathbf R$ is
	$v(\min\{s,t\})I_p$.
	
	Differentiating
	$v(t)=\chi^2(\rho(t)-c(t)^\top V(t)^{-1}c(t))$
	and using $V(t)a(t)=c(t)$ yields
	\begin{equation*}
		v'(t)
		=\chi^2g(t)^2
		\bigl(e^{\lambda t}-a(t)^\top f(t)\bigr)^2,
		\qquad t>0.
	\end{equation*}
	Proposition~\ref{prop:covariance} gives
	$v(t)\downarrow0$ as $t\downarrow0$. The displayed
	derivative and continuity imply
	$v(t)=\int_0^t v'(s)\dd s$.
	
	The centered Gaussian processes $\mathbf R_t$ and
	$\mathbf W_{v(t)}$ have the same covariance and hence
	the same finite-dimensional laws. The explicit expression
	for $\mathbf R_t$ is continuous for $t>0$. Equality of its
	laws on rational times with those of the continuous
	Brownian time change gives a continuous extension at zero.
	It follows that
	\[
	(\mathbf U_t)_{0\le t\le T}
	\overset{\mathrm{law}}=
	(\mathbf X_0+\mathbf W_{v(t)})_{0\le t\le T},
	\qquad \mathbf W\perp\mathbf X_0.
	\]
	Finally, $\int_0^t\sqrt{v'(s)}
	\dd\widetilde{\mathbf W}_s$ has the same Gaussian
	covariance $v(\min\{s,t\})I_p$. This proves the weak
	SDE representation in \eqref{eq:ve-sde}.
\end{proof}

\subsection{Proof of Proposition~\ref{prop:vp}}
\label{app:gen-vp}

\begin{proof}
	Apply the rescaling to the weak VE representation of $\mathbf U$.
	It\^o's formula gives, in law,
	\begin{align*}
		\dd\mathbf V_t
		&=-\frac{v'(t)}
		{2(1+v(t))^{3/2}}\mathbf U_t\dd t
		+\frac{\sqrt{v'(t)}}{\sqrt{1+v(t)}}
		\dd\widetilde{\mathbf W}_t =-\tfrac12\beta_{\mathrm{VP}}(t)\mathbf V_t\dd t
		+\sqrt{\beta_{\mathrm{VP}}(t)}
		\dd\widetilde{\mathbf W}_t.
	\end{align*}
	This calculation applies first on positive time
	intervals. Since
	$\int_0^T\beta_{\mathrm{VP}}(t)\dd t
	=\log(1+v(T))<\infty$, its integral form extends to
	$t=0$, where $\mathbf V_0=\mathbf X_0$.
	
	The representation \eqref{eq:ve-path} also gives
	\[
	\mathbf V_t
	\overset{\mathrm{law}}=
	\frac{\mathbf X_0+\mathbf W_{v(t)}}
	{\sqrt{1+v(t)}}.
	\]
	As $\mathbf W$ is independent of $\mathbf X_0$,
	\[
	\operatorname{Cov}(\mathbf V_t)
	=
	\frac{\operatorname{Cov}(\mathbf X_0)+v(t)I_p}
	{1+v(t)}.
	\]
	This equals $I_p$ when
	$\operatorname{Cov}(\mathbf X_0)=I_p$.
\end{proof}

\section{Time cutoff and numerical sampling}
\label{app:gen-sampling}

\subsection{The one-factor expansion}
\label{app:proj:smalltime}

\begin{proposition}[Conditional variance near zero]
	\label{app:proj:cubic}
	If $K=1$ and $g^2$ is continuously differentiable
	near zero, then
	\[
	\sigma(t)^2
	=\frac{\chi^2(\lambda-\kappa_1)^2g(0)^2}{12}
	t^3+O(t^4)
	\qquad(t\downarrow0).
	\]
\end{proposition}

\begin{proof}
	Set $d=\lambda-\kappa_1$ and change variables to
	$u=t-s$ in \eqref{app:proj:residualprojection}.
	The coefficient minimizing its second integral is
	\[
	a_t=
	\frac{\int_0^t g(t-u)^2
		e^{-(\lambda+\kappa_1)u}\dd u}
	{\int_0^t g(t-u)^2e^{-2\kappa_1u}\dd u}.
	\]
	Uniformly over $0\le u\le t$,
	$g(t-u)^2=g(0)^2+O(t)$ and
	$e^{-du}=1-du+O(t^2)$. Consequently,
	\begin{align*}
		a_t-1
		&=
		\frac{\int_0^t g(t-u)^2e^{-2\kappa_1u}
			(e^{-du}-1)\dd u}
		{\int_0^t g(t-u)^2e^{-2\kappa_1u}\dd u} =-\frac d2t+O(t^2).
	\end{align*}
	It follows that, uniformly for $u\in[0,t]$,
	\[
	e^{-\lambda u}-a_te^{-\kappa_1u}
	=-d(u-t/2)+O(t^2).
	\]
	Substitution into \eqref{app:proj:residualprojection}
	now gives
	\begin{align*}
		\sigma(t)^2
		&=\chi^2g(0)^2d^2
		\int_0^t(u-t/2)^2\dd u+O(t^4) =\frac{\chi^2g(0)^2d^2}{12}t^3+O(t^4),
	\end{align*}
	as required.
\end{proof}

\subsection{Proof of Proposition~\ref{app:gen-cutoff}}
\label{app:gen-endpoints-section}

\begin{proof}
	Let $\boldsymbol\xi_t$ have its exact forward
	distribution. The pair
	$(D_t^*(\boldsymbol\xi_t),\mathbf X_0)$ couples
	$\mathcal L(D_t^*(\boldsymbol\xi_t))$ and $P_0$.
	Thus
	\[
	W_2^2\!\left(
	\mathcal L(D_t^*(\boldsymbol\xi_t)),P_0
	\right)
	\le
	\mathbb E\|D_t^*(\boldsymbol\xi_t)-\mathbf X_0\|^2.
	\]
	Conditional expectation minimizes mean-square error,
	and $\boldsymbol\xi_t/\alpha(t)$ is an admissible
	estimator of $\mathbf X_0$. By
	\eqref{app:gen-residual-law},
	\begin{align*}
		\mathbb E\|D_t^*(\boldsymbol\xi_t)-\mathbf X_0\|^2
		&\le
		\mathbb E\|\boldsymbol\xi_t/\alpha(t)
		-\mathbf X_0\|^2 =\sigma_{\mathrm{eff}}(t)^2
		\mathbb E\|\boldsymbol\zeta_t\|^2 = p\,\sigma_{\mathrm{eff}}(t)^2.
	\end{align*}
	Combining the inequalities proves the claim.
\end{proof}

\subsection{Reverse dynamics and Heun integration}
\label{app:gen-heun}

For the exact score and exact initial law on
$0<t_-<t_+$, classical diffusion time reversal
\cite{haussmann1986time} gives, with $t=t_+-\tau$,
\begin{align*}
	& \dd\overline{\mathbf Z}_\tau =\Bigl[
	-(F\otimes I_p)\overline{\mathbf Z}_\tau
	+g(t)^2(qq^\top\otimes I_p)
	\nabla_{\mathbf z}
	\log p_t(\overline{\mathbf Z}_\tau)
	\Bigr]\dd\tau + g(t)(q\otimes I_p)\dd\overline{\mathbf B}_\tau,
	\qquad
	\overline{\mathbf Z}_0\sim p_{t_+}.
\end{align*}
Here $p_t$ is smooth and positive on the retained
positive-time interval because $C(t)\succ0$.
In the numerical sampler, the exact score is replaced
by $\mathbf S_\theta$, computed from the denoiser and
the analytic auxiliary term.

Let $t_0=t_+>\cdots>t_N=t_-$ be the sampling grid.
Define $h_n=t_{n+1}-t_n<0$ and
\[
b_\theta(\mathbf z,t)
=(F\otimes I_p)\mathbf z
-g(t)^2(qq^\top\otimes I_p)
\mathbf S_\theta(\mathbf z,t).
\]
For independent
$\boldsymbol\zeta_n\sim\mathcal N(0,I_p)$, the
implemented noise increment is
\[
\boldsymbol\Delta_n
=(q\otimes I_p)
\sqrt{\frac{g(t_n)^2+g(t_{n+1})^2}{2}|h_n|}
\,\boldsymbol\zeta_n.
\]
The predictor and corrector use this same increment:
\begin{align*}
	\mathbf Z_{n+1}^{\mathrm{pred}}
	&=\mathbf Z_n+h_nb_\theta(\mathbf Z_n,t_n)
	+\boldsymbol\Delta_n,\\
	\mathbf Z_{n+1}
	&=\mathbf Z_n+\frac{h_n}{2}
	\Bigl[
	b_\theta(\mathbf Z_n,t_n)
	+b_\theta(\mathbf Z_{n+1}^{\mathrm{pred}},t_{n+1})
	\Bigr]
	+\boldsymbol\Delta_n.
\end{align*}
The scalar variance in $\boldsymbol\Delta_n$ is the
trapezoidal approximation to
$\int_{t_{n+1}}^{t_n}g(s)^2\dd s$.
The grid is uniform in effective log noise:
\[
\log\sigma_{\mathrm{eff}}(t_n)
=\left(1-\frac nN\right)
\log\sigma_{\mathrm{eff}}(t_+)
+\frac nN\log\sigma_{\mathrm{eff}}(t_-).
\]
The increasing effective-noise map is inverted
numerically to obtain the times.

On the final interval, the implementation sets
$\boldsymbol\Delta_{N-1}=0$ while retaining both
drift evaluations. A separate denoiser call at $t_-$
produces the clean-data prediction. Thus $N$
intervals use $2N+1$ main-network evaluations;
active autoguidance also evaluates the guide at
each of these calls.

%%%%%%%%%%%%%%%%%%%%%%%%%%%%%%%

%% \input{sections/experimental_details}

\FloatBarrier
\section{Implementation details}
\label{app:exp-implementation}

All reported configurations use one auxiliary factor ($K=1$), horizon
$T=1$, and projection ridge $\delta=10^{-10}$. Kernel parameters are
selected before training and are shared across datasets for a given
Hurst parameter. The diffusion power and admissible time interval are
determined separately for each dataset.

The randomized search samples candidate rates and coefficient
magnitudes on logarithmic scales. For each $(\lambda,\chi,\kappa_1)$,
the projection determines $\psi_1$; candidates are then checked
against the bounds in Table~\ref{app:exp-search}. The search stops
after finding suitable parameters and does not claim a global optimum.

\subsection{Diffusion power and covariance tables}

For the image experiments, let $s_0=0.008$ and define
\begin{equation*}
	\beta_0(t)=\frac{\pi}{T(1+s_0)}
	\tan\!\left(\frac{\pi}{2}\frac{t/T+s_0}{1+s_0}\right),
	\; 0\leq t<T.
\end{equation*}
We clip this power profile to
$[\beta_{\min},\beta_{\max}]$, using $\beta_{\max}$ at $T$.
The dataset-specific limits appear below. The volatility is
normalized as
\begin{equation}
	g(t)^2=
	\frac{v_\star\beta(t)}
	{\int_0^T\widehat G(T-r)^2\beta(r)\,\mathrm{d}r},
	\quad v_\star=c_{xx}(T).
	\label{app:exp-power-normalization}
\end{equation}
Here, $v_\star$ controls the terminal variance of the primary
innovation.

We integrate the covariance ODE by fourth-order Runge--Kutta on
$2049$ uniformly spaced points in $[0,T]$. After each step, the covariance is symmetrized and negative eigenvalues are clipped to zero. The auxiliary inverse uses a relative eigenvalue floor of $10^{-5}$; this floor is inactive for the positive scalar $C_{yy}(t)$ in the reported one-factor runs.

The base cutoff $t_{\rm safe}$ is the first tabulated time $t\geq10^{-3}$ satisfying $\sigma(t)\geq\sigma_{\min}$ and $\operatorname{cond}_2(C_{yy}(t))\leq10^6$. Table~\ref{app:exp-safe-times} gives the resulting cutoffs. For the pretrained models, the supported effective-noise window can further restrict the sampling interval.

\subsection{Training, sampling, and evaluation conventions}

All denoisers use $\sigma_{\rm data}=0.5$. In optimized runs,
$\log\sigma_{\rm eff}(t)$ is drawn from a Gaussian distribution and
clipped to the available noise interval. The denoising loss uses
EDM weighting, with a dataset-specific upper cap on $w(t)$.
Optimization uses AdamW with moment coefficients $(0.9,0.99)$,
numerical stabilizer $10^{-8}$, zero weight decay, and gradient-norm
limit $1$. Unless stated otherwise, the random seed is $1234$.

Sampling times are spaced uniformly in $\log\sigma_{\rm eff}$ and
obtained by numerical inversion of the tabulated noise map. The
shared-noise Heun update omits the random increment on its final
interval and makes one additional denoiser evaluation at the lower
endpoint. Hence $N$ intervals require $2N+1$ main-network
evaluations; active autoguidance requires the same number of guide
evaluations. For pretrained models, the supported noise window is
imposed by selecting physical time endpoints.

FID is computed from \texttt{torch-fidelity} Inception features
\citep{heusel2017gans,obukhov2020torchfidelity}. Each evaluation
uses $50{,}000$ generated images, converted to unsigned eight-bit
pixels and stored without lossy compression. The real-image
references are specified for each dataset below.

\subsection{MNIST experimental details}
\label{app:exp-mnist}

MNIST images are used at their native $28\times28$ grayscale
resolution, scaled to $[-1,1]$. The class-conditional denoiser is
trained from scratch without augmentation or additional guidance.
Table~\ref{app:exp-mnist-settings} specifies the best reported
configuration. Its model-factory preset is \texttt{mn\_large6m}.

The selected checkpoint uses the final exponential moving average
(EMA) weights. FID compares $50{,}000$ generated samples with the
first $50{,}000$ images of the training split; grayscale reference
images are replicated across RGB channels for feature extraction.

\subsection{CIFAR-10 experimental details}
\label{app:exp-cifar}

The class-conditional network is trained from scratch on
$32\times32$ RGB images. Training uses random horizontal flips and
no additional guidance. The global batch of $512$ consists of
$128$ images on each of four GPUs, without gradient accumulation.
Consequently, $100{,}000$ updates process $51.2$ million image
presentations.

Generated class labels are balanced. FID uses all $50{,}000$
CIFAR-10 training images as its real reference.

\subsection{ImageNet64 experimental details}
\label{app:exp-imagenet64}

We fine-tune the pretrained class-conditional EDM2 small main
network while keeping its conditional EDM2 XS autoguide frozen.
The checkpoint suffixes are \texttt{1073741-0.045} for the main
network and \texttt{0134217-0.110} for the guide. Images are
center-cropped according to the donor preprocessing convention,
stored as $64\times64$ RGB arrays, and assigned labels in the
donor's class order.

Fine-tuning processes $4.096$ million image presentations. The
effective-noise window applies to both training and sampling;
sampling endpoints are selected through $\sigma_{\rm eff}(t)$.
For FID, balanced generated class labels are compared with the
$50{,}000$ ImageNet validation images.

\subsection{ImageNet512 experimental details}
\label{app:exp-imagenet512}

The EDM2 small main network, its conditional EDM2 XS autoguide,
and the \texttt{sd-vae-ft-mse} decoder are frozen. Their checkpoint
suffixes are \texttt{2147483-0.070} and
\texttt{0134217-0.125}, respectively, for the main network and
guide. Sampling takes place in the donor's $4\times64\times64$
latent space; the decoder maps each latent to a
$512\times512$ RGB image. No optimizer updates are performed.

The real-reference selection samples without replacement using
seed $1234$. Its balanced class distribution matches that of the
generated labels. VAE outputs are converted to unsigned eight-bit
RGB images before FID evaluation. With $50$ Heun intervals and
endpoint denoising, guided generation evaluates each frozen
network $101$ times per sample, in addition to VAE decoding.

\FloatBarrier

%%%%%%%%%% Tables %%%%%%%%%%%%%%

\begin{table*}[t]
	\setlength{\belowcaptionskip}{6pt}
	\caption{Parameter-search ranges and acceptance criteria. The
		covariance checks in this search are distinct from the dataset-specific
		sampling cutoffs in Table~\ref{app:exp-safe-times}.}
	\label{app:exp-search}
	\centering
	\small
	\resizebox{\textwidth}{!}{%
	\begin{tabular}{@{}ll@{}}
		\toprule
		Quantity & Range or criterion\\
		\midrule
		$\lambda$, $|\chi|$, $\kappa_1$
		& $[2,15]$, $[0.8,5]$, $[0.1,10]$\\
		Separation of $\lambda$ and $\kappa_1$
		& Ratio at least $1.25$ in either direction\\
		Squared kernel error
		& At most $0.1$\\
		Coefficient magnitudes
		& $|\varepsilon|\leq1.5$, $|\psi_1|\leq10$\\
		Covariance condition numbers
		& $\operatorname{cond}_2(C_{yy})\leq500$,
		$\operatorname{cond}_2(C)\leq10^4$\\
		Projection Gram condition number
		& At most $10^8$\\
		Search-time lower bound and residual threshold
		& $10^{-4}$ and $\sigma(t)\geq0.01$\\
		Terminal effective noise; clipped training-noise probability
		& At most $5000$; at most $0.05$\\
		Search diffusion power; target $c_{xx}(T)$
		& Linear from $0.1$ to $10$; $1$\\
		Candidate sampling
		& $2000$ initial trials and six rounds of $1000$ trials\\
		Covariance grids; random seed
		& $1025$ search points, $8193$ verification points; $1234$\\
		\bottomrule
	\end{tabular}}%
\end{table*}

\begin{table}[t]
	\setlength{\belowcaptionskip}{6pt}
	\caption{Selected kernel parameters for the best reported configurations.
		The coefficient $\psi_1$ is obtained by projection, and
		$\varepsilon=\chi+\psi_1$. Values are rounded for display.}
	\label{app:exp-parameters}
	\centering
	\begin{tabular}{@{}lrr@{}}
		\toprule
		Quantity & $H=0.7$ & $H=0.9$\\
		\midrule
		Primary rate $\lambda$ & $8.98007$ & $8.85873$\\
		Primary coefficient $\chi$ & $-2.76950$ & $-2.66401$\\
		Auxiliary rate $\kappa_1$ & $0.94001$ & $0.71990$\\
		Projected coefficient $\psi_1$ & $1.86747$ & $1.54188$\\
		Kernel value $\varepsilon$ at zero & $-0.90203$ & $-1.12213$\\
		Squared $L^2(0,T)$ error & $0.09996$ & $0.09759$\\
		\bottomrule
	\end{tabular}
\end{table}

\begin{table}[t]
	\setlength{\belowcaptionskip}{6pt}
	\caption{Residual thresholds, terminal primary-variance targets,
		and base covariance cutoffs.}
	\label{app:exp-safe-times}
	\centering
	\begin{tabular}{@{}lrrr@{}}
		\toprule
		Dataset & $\sigma_{\min}$ & $v_\star$ & $t_{\rm safe}$\\
		\midrule
		MNIST & $0.01$ & $1$ & $0.04590$\\
		CIFAR-10 & $0.002$ & $1$ & $0.01709$\\
		ImageNet64 & $0.005$ & $1.5$ & $0.02539$\\
		ImageNet512 & $0.01$ & $1$ & $0.04297$\\
		\bottomrule
	\end{tabular}
\end{table}

\begin{table}[t]
	\setlength{\belowcaptionskip}{6pt}
	\caption{MNIST configuration at $H=0.7$. Kernel parameters and the
		base cutoff are given in Tables~\ref{app:exp-parameters}
		and~\ref{app:exp-safe-times}.}
	\label{app:exp-mnist-settings}
	\centering
	\small
	\begin{tabular}{@{}p{0.39\columnwidth}p{0.55\columnwidth}@{}}
		\toprule
		Setting & Value\\
		\midrule
		Network
		& $6$M-parameter U-Net; width $32$; channel multipliers $(1,2,4)$\\
		Blocks and attention
		& Three residual blocks per resolution; attention at $14$ and $7$; $32$ channels per head\\
		Dropout; class dropout
		& $0$; $0$\\
		Clip of diffusion power $\beta$
		& $(\beta_{\min},\beta_{\max}) = (0.1,10)$\\
		Training log-noise mean; standard deviation
		& $-1.2$; $1.5$\\
		Loss-weight cap
		& $w(t)\leq1000$\\
		Updates; batch size
		& $50{,}000$; $1024$\\
		Learning rate
		& OneCycle, peak $2\times10^{-4}$; rising fraction $0.1$; initial and final divisors $25$ and $10^4$\\
		EMA decay
		& $0.9995$\\
		Sampling
		& $200$ stochastic Heun intervals; guidance strength $1$\\
		\bottomrule
	\end{tabular}
\end{table}

\begin{table}[t]
	\setlength{\belowcaptionskip}{6pt}
	\caption{CIFAR-10 configuration at $H=0.7$. Kernel parameters and
		the base cutoff are given in Tables~\ref{app:exp-parameters}
		and~\ref{app:exp-safe-times}.}
	\label{app:exp-cifar-settings}
	\centering
	\small
	\begin{tabular}{@{}p{0.39\columnwidth}p{0.55\columnwidth}@{}}
		\toprule
		Setting & Value\\
		\midrule
		Network
		& $60$M-parameter U-Net; width $128$; channel multipliers $(2,2,2)$\\
		Blocks and attention
		& Four residual blocks per resolution; attention at $16$\\
		Dropout; class dropout
		& $0.15$; $0$\\
		Clip of diffusion power $\beta$
		& $(\beta_{\min},\beta_{\max}) = (0.1,30)$\\
		Training log-noise mean; standard deviation
		& $-1.2$; $1.2$\\
		Loss-weight cap
		& $w(t)\leq10{,}000$\\
		Updates; global batch
		& $100{,}000$; $512$\\
		Learning rate
		& Linear warmup for $20{,}000$ updates to $8\times10^{-4}$, then constant\\
		EMA decay
		& $0.99929$\\
		Sampling
		& $50$ stochastic Heun intervals; guidance strength $1$\\
		\bottomrule
	\end{tabular}
\end{table}

\begin{table}[t]
	\setlength{\belowcaptionskip}{6pt}
	\caption{ImageNet64 configuration at $H=0.9$. Kernel parameters and
		the base cutoff are given in Tables~\ref{app:exp-parameters}
		and~\ref{app:exp-safe-times}.}
	\label{app:exp-imagenet64-settings}
	\centering
	\small
	\begin{tabular}{@{}p{0.39\columnwidth}p{0.55\columnwidth}@{}}
		\toprule
		Setting & Value\\
		\midrule
		Training
		& Fine-tune main network; freeze conditional guide\\
		Augmentation; class dropout
		& Random horizontal flip; $0$\\
		Clip of diffusion power $\beta$
		& $(\beta_{\min},\beta_{\max}) = (0.1,30)$\\
		Effective-noise window
		& $\sigma_{\rm eff}\in[0.007,80]$\\
		Training log-noise mean; standard deviation
		& $-0.8$; $1.6$\\
		Loss-weight cap
		& $w(t)\leq10{,}000$\\
		Updates; global batch
		& $1000$; $4096$\\
		Microbatch; accumulation
		& $64$ per GPU on four GPUs; $16$ microbatches per update\\
		Learning rate
		& OneCycle, peak $4\times10^{-5}$; rising fraction $0.1$; initial and final divisors $4$ and $1.25$\\
		EMA decay
		& $0.9999$\\
		Sampling
		& $100$ stochastic Heun intervals; autoguidance strength $\omega=1.45$\\
		\bottomrule
	\end{tabular}
\end{table}

\begin{table}[t]
	\setlength{\belowcaptionskip}{6pt}
	\caption{Frozen ImageNet512 configuration at $H=0.9$.
		Kernel parameters and the base cutoff are given in
		Tables~\ref{app:exp-parameters} and~\ref{app:exp-safe-times}.}
	\label{app:exp-imagenet512-settings}
	\centering
	\small
	\begin{tabular}{@{}p{0.39\columnwidth}p{0.55\columnwidth}@{}}
		\toprule
		Setting & Value\\
		\midrule
		Networks and decoder
		& Frozen EDM2 small, conditional EDM2 XS, and VAE\\
		Diffusion state
		& $4\times64\times64$ latent; $1000$ classes\\
		Clip of diffusion power $\beta$
		& $(\beta_{\min},\beta_{\max}) = (0.1,10)$\\
		Effective-noise window
		& $\sigma_{\rm eff}\in[0.02,80]$\\
		Sampling
		& $50$ stochastic Heun intervals; autoguidance strength $\omega=1.85$\\
		Generated labels
		& Balanced across classes\\
		Real reference
		& $50{,}000$ training images, sampled at $50$ per class\\
		\bottomrule
	\end{tabular}
\end{table}

\FloatBarrier

%%%%%% Figures %%%%%%%%

\begin{figure*}[t]
	\centering
	\includegraphics[width=0.9\linewidth]{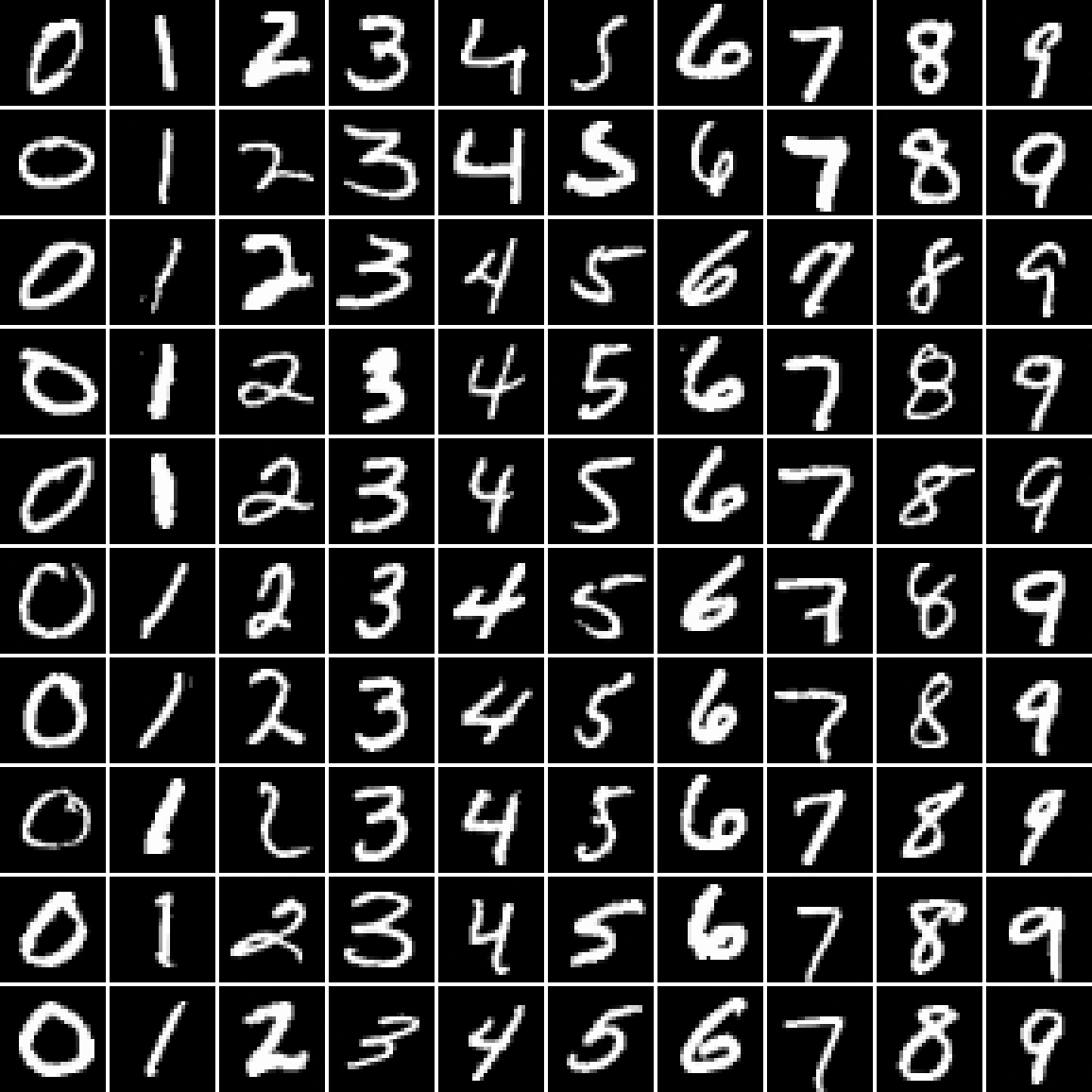}
	\caption{Generated MNIST images at $H=0.7$, using class conditioning without additional guidance. FID $=0.547$.}
\end{figure*}

\begin{figure*}[t]
	\centering
	\includegraphics[width=0.9\linewidth]{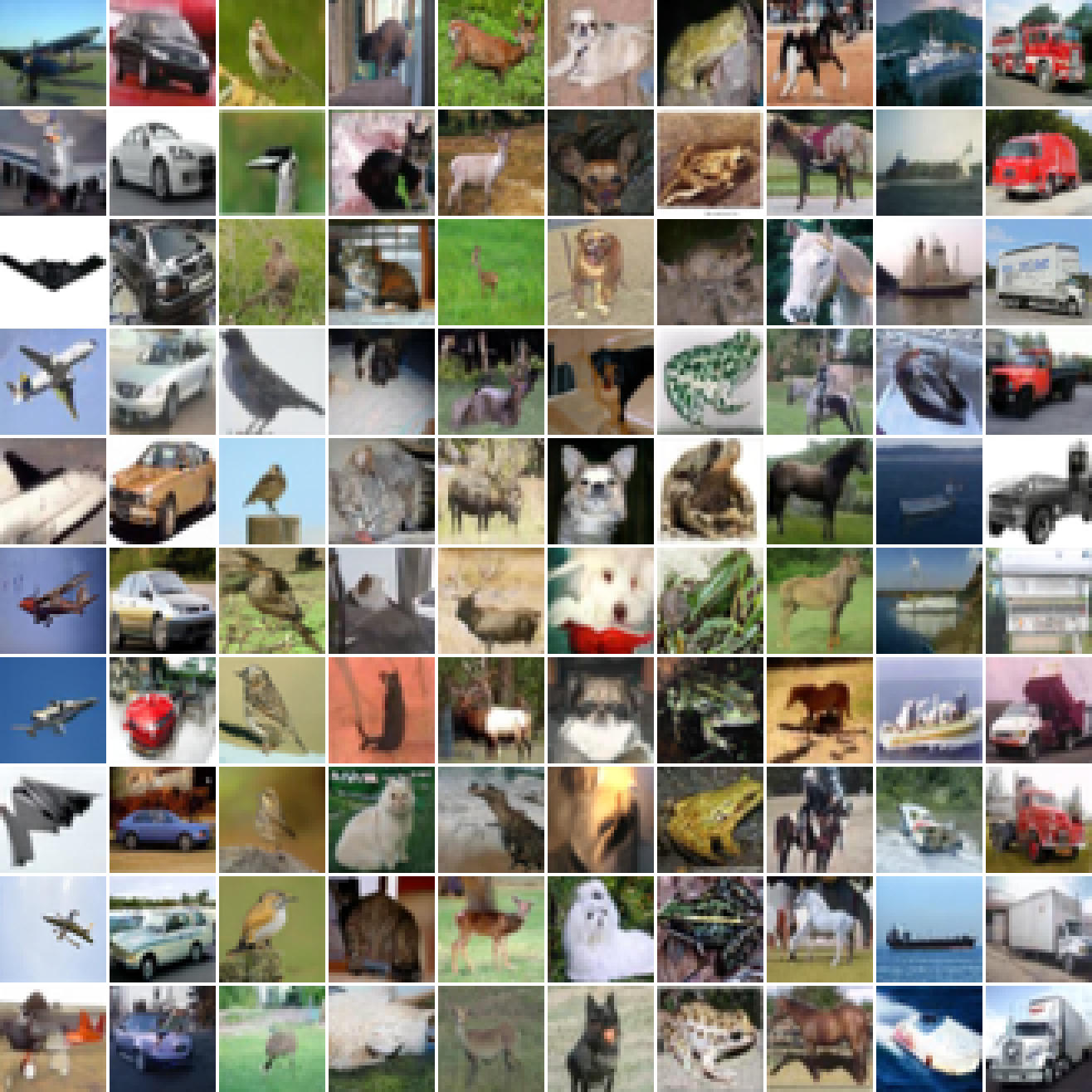}
	\caption{Generated CIFAR-10 images at $H=0.7$, using class conditioning without additional guidance. FID $=1.860$.}
\end{figure*}

\begin{figure*}[t]
	\centering
	\includegraphics[width=0.9\linewidth]{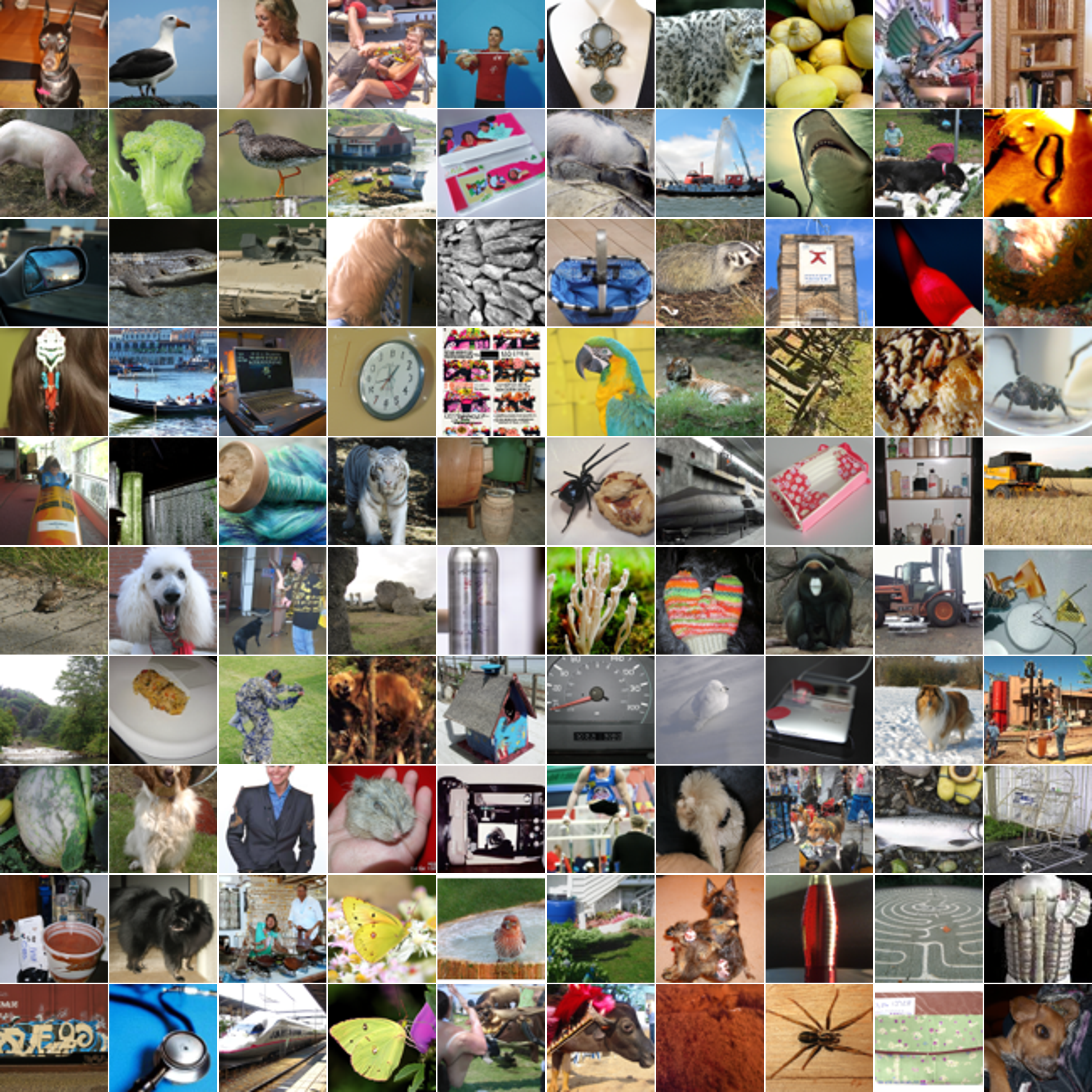}
	\caption{Generated ImageNet64 images at $H=0.9$ and autoguidance strength $\omega=1.45$. FID $=1.912$.}
\end{figure*}

\begin{figure*}[t]
	\centering
	\includegraphics[width=0.242\linewidth]{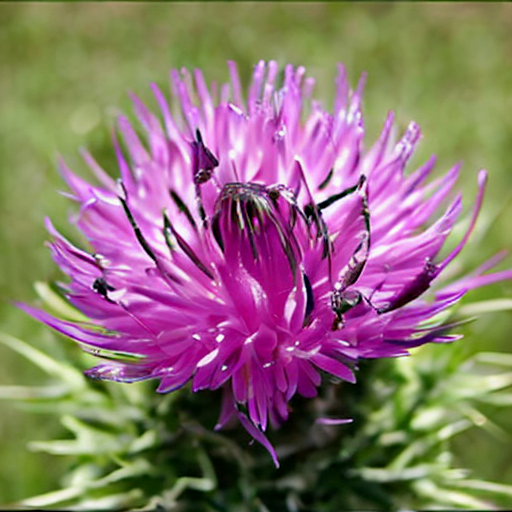}
	\hfill
	\includegraphics[width=0.242\linewidth]{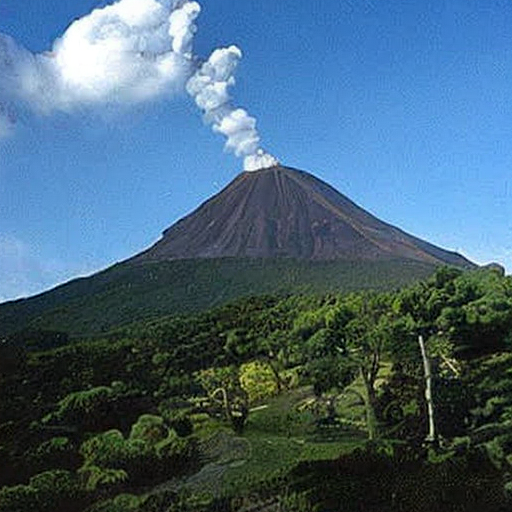}
	\hfill
	\includegraphics[width=0.242\linewidth]{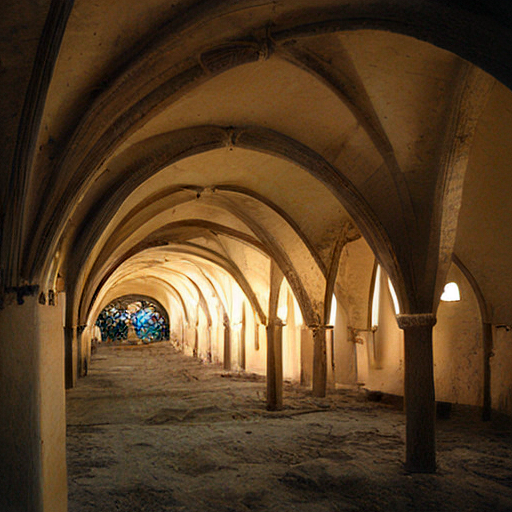}
	\hfill
	\includegraphics[width=0.242\linewidth]{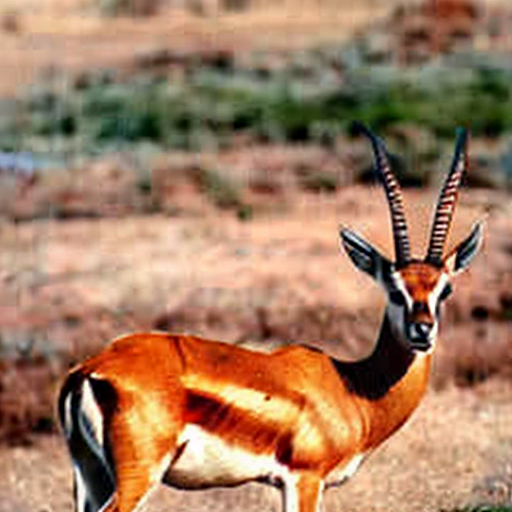}
	\\[2pt]
	\includegraphics[width=0.242\linewidth]{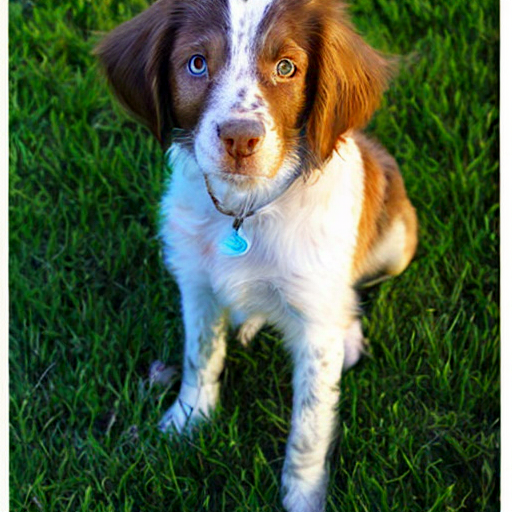}
	\hfill
	\includegraphics[width=0.242\linewidth]{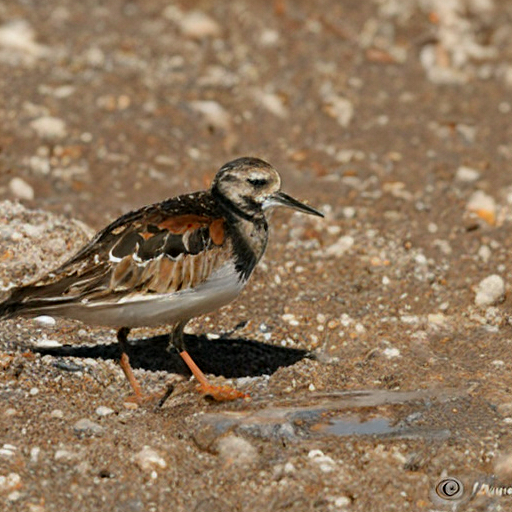}
	\hfill
	\includegraphics[width=0.242\linewidth]{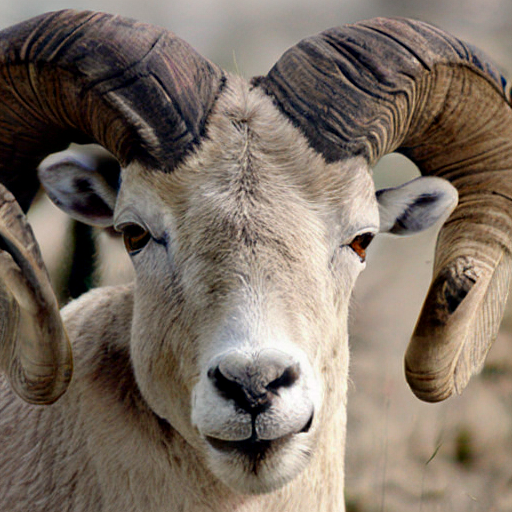}
	\hfill
	\includegraphics[width=0.242\linewidth]{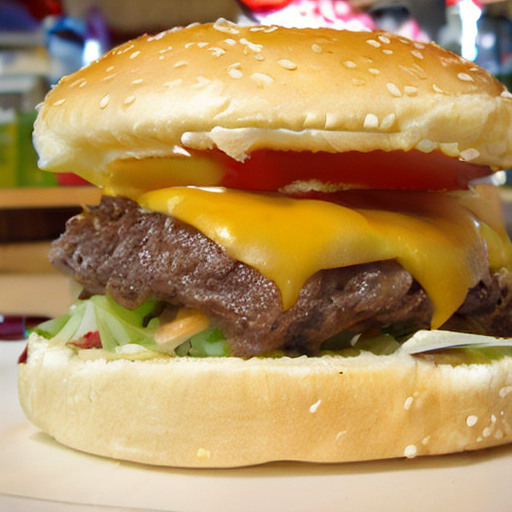}
	\\[2pt]
	\includegraphics[width=0.242\linewidth]{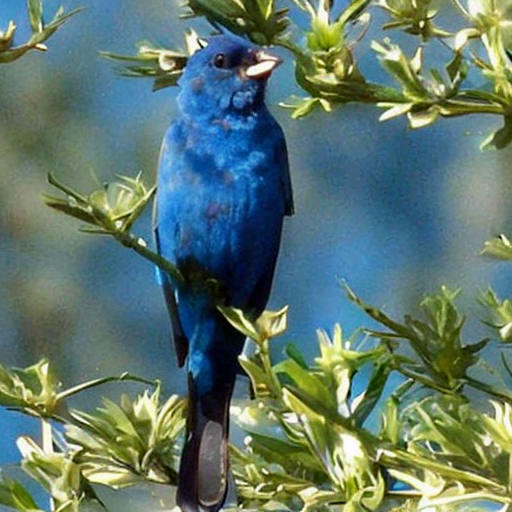}
	\hfill
	\includegraphics[width=0.242\linewidth]{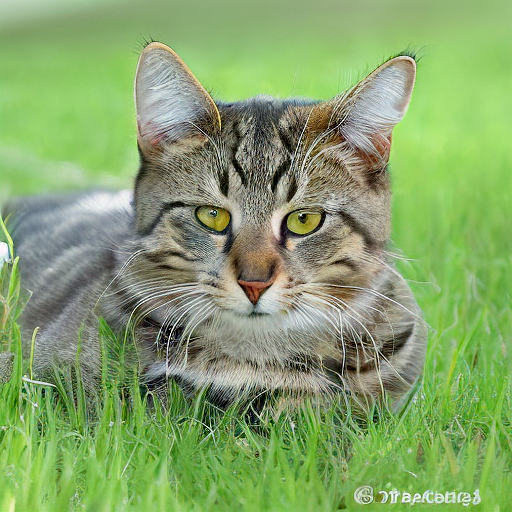}
	\hfill
	\includegraphics[width=0.242\linewidth]{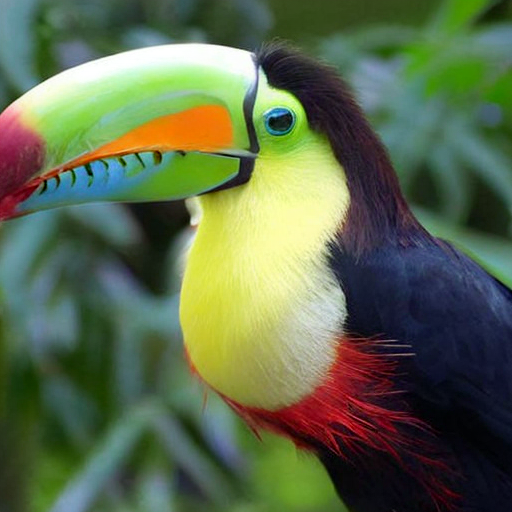}
	\hfill
	\includegraphics[width=0.242\linewidth]{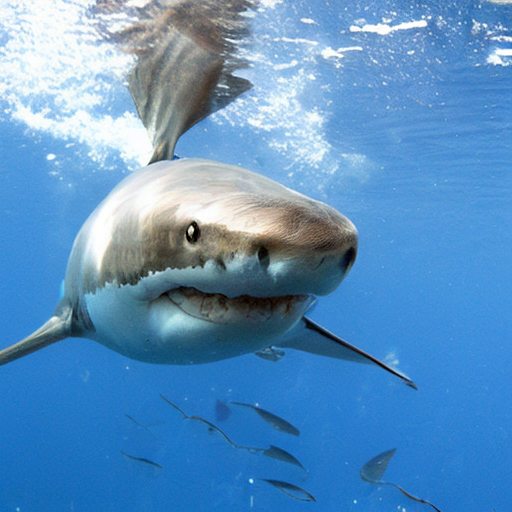}
	\\[2pt]
	\includegraphics[width=0.242\linewidth]{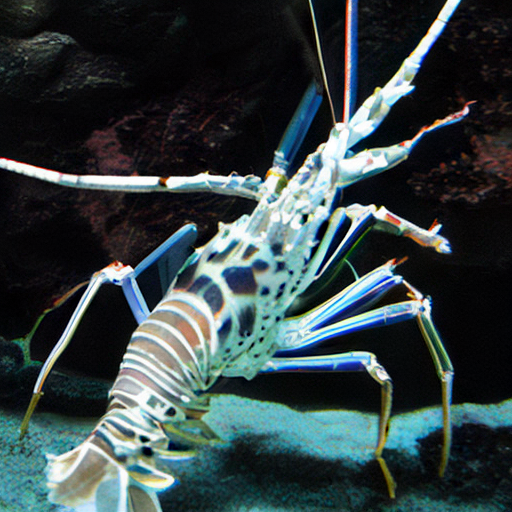}
	\hfill
	\includegraphics[width=0.242\linewidth]{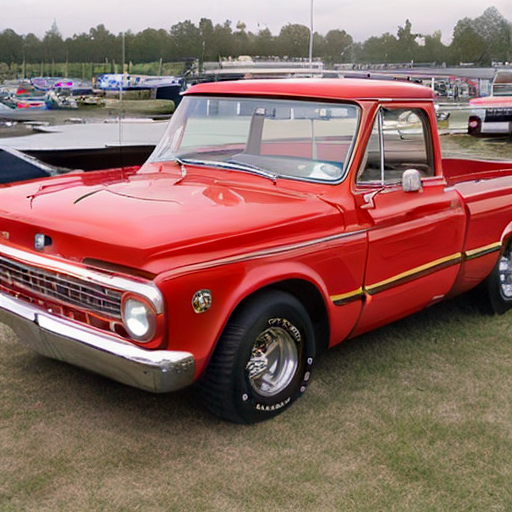}
	\hfill
	\includegraphics[width=0.242\linewidth]{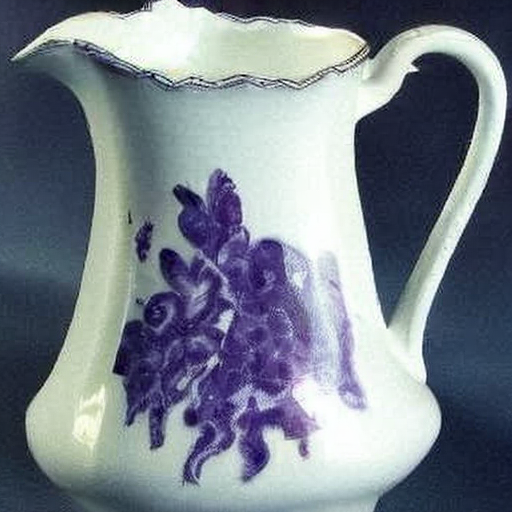}
	\hfill
	\includegraphics[width=0.242\linewidth]{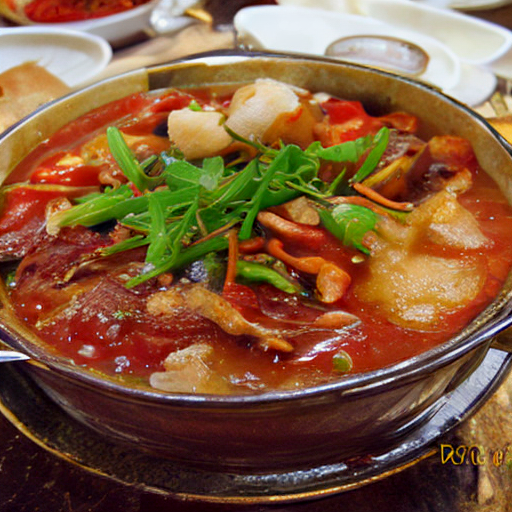}
	\\[2pt]
	\includegraphics[width=0.242\linewidth]{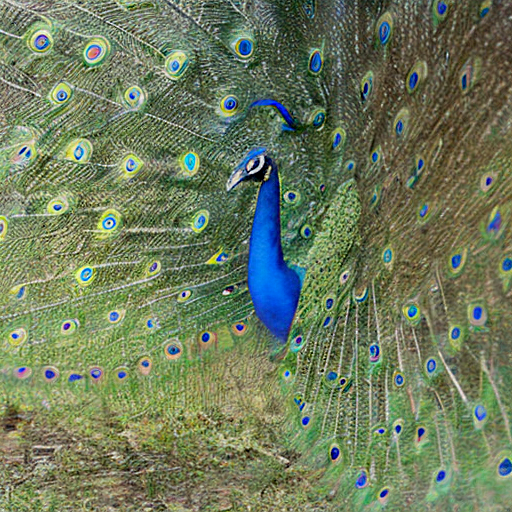}
	\hfill
	\includegraphics[width=0.242\linewidth]{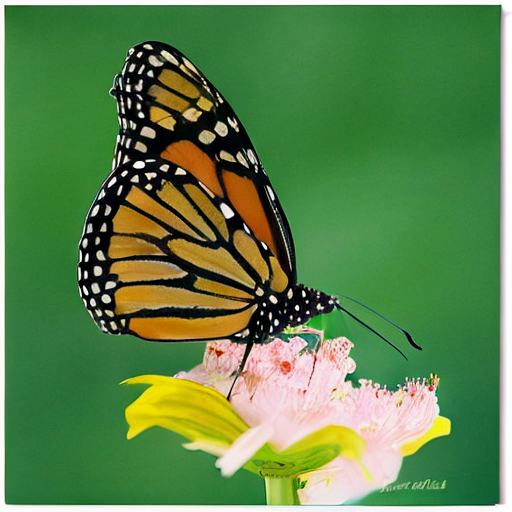}
	\hfill
	\includegraphics[width=0.242\linewidth]{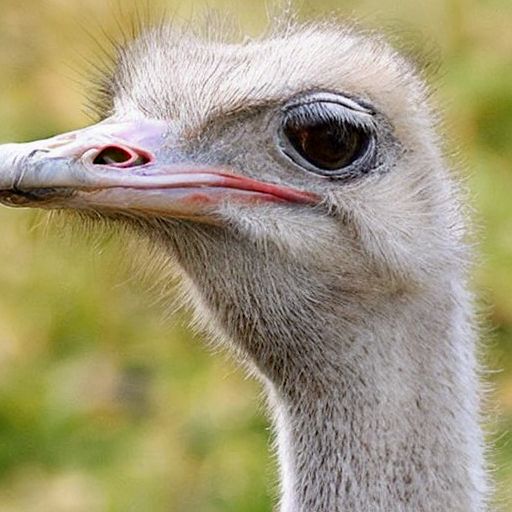}
	\hfill
	\includegraphics[width=0.242\linewidth]{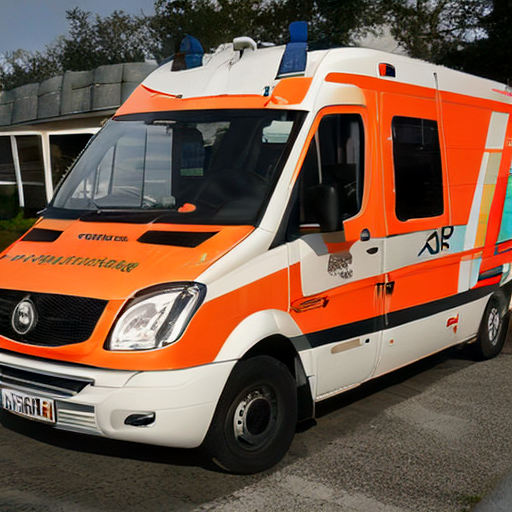}
	\caption{Additional generated ImageNet512 images at $H=0.9$ and
		autoguidance strength $\omega=1.85$. FID $=1.903$.}
	\label{app:fig-imagenet512-samples}
\end{figure*}

\FloatBarrier

\section{Comparison with published FID results}
\label{app:literature-comparison}

Tables~\ref{tab:literature-mnist}--\ref{tab:literature-imagenet64}
place our FIDs alongside selected published results. We emphasize
class-conditional diffusion models where available and include a
small number of results from other generative approaches for context.

NFE denotes denoiser evaluations per generated image. When two
networks are evaluated for guidance, we list their counts separately
as main$+$guide. Classifier or discriminator evaluations and VAE
decoding are additional costs. A dash indicates that it is not reported.

Our CIFAR-10 FID is lower than the GFDM, Bit Diffusion, and DDSM values, while EDM and several guided or distilled systems report lower FIDs. Our ImageNet64 FID is lower than the Improved Diffusion, U-ViT, Bit Diffusion, and ADM values. It is above the reported EDM and EDM2 values. The experiment shows that the Volterra sampler can use an EDM2 denoiser after $1000$ fine-tuning updates; it does not establish an improvement over the donor sampler.

\begin{table}[t]
	\setlength{\belowcaptionskip}{6pt}
	\caption{MNIST FIDs. C and U denote class-conditional and unconditional generation. The VP and GFDM entries come from the same study \cite{nobis2024generative}. DLIM evaluates $5000$ generated and $5000$ real images, whereas our evaluation uses $50{,}000$ generated images. The unconditional entries are included as broader context.}
	\label{tab:literature-mnist}
	\centering
	\small
	\begin{tabularx}{0.6\columnwidth}{@{}Xcrr@{}}
		\toprule
		Method & Setting & FID $\downarrow$ & NFE\\
		\midrule
		VP, retrained~\cite{nobis2024generative}
		& C & $1.440$ & --\\
		GFDM, $H=0.9$, $K=3$~\cite{nobis2024generative}
		& C & $0.720$ & --\\
		cU-Net with DDIM~\cite{calvo2024missing}
		& U & $2.550$ & --\\
		DLIM, $\alpha=1.7$~\cite{shariatian2025heavy}
		& U & $2.810$ & $25$\\
		\midrule
		Ours, $H=0.7$
		& C & $0.547$ & $401$\\
		\bottomrule
	\end{tabularx}
\end{table}

\begin{table}[t]
	\setlength{\belowcaptionskip}{6pt}
	\caption{Class-conditional CIFAR-10 FIDs. Bit Diffusion models pixels as analog bits. DDSM adapts network width across sampling steps; its result is reported in the supplement to \citet{yang2024denoising}. GFDM's $3.770$ is its main-table result; its appendix also reports an averaged FID of $3.700$. EDM-G++ evaluates an additional discriminator, excluded from denoiser NFE. SiD distills a pretrained teacher.}
	\label{tab:literature-cifar}
	\centering
	\small
	\begin{tabularx}{0.6\columnwidth}{@{}Xrr@{}}
		\toprule
		Method & FID $\downarrow$ & NFE\\
		\midrule
		GFDM~\cite{nobis2024generative}
		& $3.770$ & --\\
		Bit Diffusion~\cite{chen2023analog}
		& $2.720$ & --\\
		DDSM~\cite{yang2024denoising}
		& $2.520$ & --\\
		EDM~\cite{karras2022elucidating}
		& $1.790$ & $35$\\
		EDM with LoRA conditioning~\cite{choi2024lora}
		& $1.750$ & $35$\\
		EDM-G++~\cite{kim2023refining}
		& $1.640$ & $35$\\
		SiD~\cite{zhou2024sid}
		& $1.710\pm0.011$ & $1$\\
		\midrule
		Ours, trained from scratch
		& $1.860$ & $101$\\
		\bottomrule
	\end{tabularx}
\end{table}

\begin{table}[t]
	\setlength{\belowcaptionskip}{6pt}
	\caption{Class-conditional ImageNet64 FIDs. Improved Diffusion uses $250$ sampling steps for its reported $2.920$ result. Bit Diffusion models analog bits, and U-ViT-L/4 uses a vision transformer denoiser.}
	\label{tab:literature-imagenet64}
	\centering
	\small
	\begin{tabularx}{0.6\columnwidth}{@{}Xrr@{}}
		\toprule
		Method & FID $\downarrow$ & NFE\\
		\midrule
		Bit Diffusion~\cite{chen2023analog}
		& $4.840$ & --\\
		U-ViT-L/4~\cite{bao2023all}
		& $4.260$ & --\\
		Improved Diffusion~\cite{nichol2021improved}
		& $2.920$ & $250$\\
		ADM with dropout~\cite{dhariwal2021diffusion}
		& $2.070$ & $250$\\
		EDM, stochastic sampler~\cite{karras2022elucidating}
		& $1.360$ & $511$\\
		RIN~\cite{jabri2023scalable}
		& $1.230$ & $1000$\\
		EDM2-S, no guidance~\cite{karras2024autoguidance}
		& $1.580$ & $63$\\
		EDM2-S with autoguidance~\cite{karras2024autoguidance}
		& $1.010$ & $63+63$\\
		SiD~\cite{zhou2024sid}
		& $1.524\pm0.009$ & $1$\\
		\midrule
		Ours, EDM2-S fine-tuning
		& $1.912$ & $201+201$\\
		\bottomrule
	\end{tabularx}
\end{table}

\FloatBarrier
{\small
\bibliographystyle{tmlr}
\bibliography{refs}
}
\end{document}